\documentclass[10pt, a4paper, twocolumn]{article}

\usepackage[utf8]{inputenc} 
\usepackage[T1]{fontenc}    
\usepackage{lmodern}
\usepackage[hidelinks]{hyperref}       
\usepackage{url}            
\usepackage{booktabs}       
\usepackage{amsfonts}       
\usepackage{nicefrac}       
\usepackage{microtype}      
\usepackage{authblk}        
\usepackage{natbib}
\renewcommand{\cite}{\citep}

\usepackage{import}
\usepackage{mathtools} 
\usepackage{enumerate}
\usepackage{amsmath}
\usepackage{amsthm}

\usepackage{multirow} 
\usepackage{graphicx}
\usepackage{subfigure}
\usepackage{caption}
\usepackage{color}
\usepackage[dvipsnames]{xcolor}
\ifpdf
  \DeclareGraphicsExtensions{.eps,.pdf,.png,.jpg}
\else
  \DeclareGraphicsExtensions{.eps}
\fi

\usepackage{amsthm}
\theoremstyle{plain}
\newtheorem{theorem}{Theorem}

\newtheorem{lemma}{Lemma}
\newtheorem{assumption}{Assumption}

\usepackage{algorithm}
\usepackage[compatible]{algpseudocode}
\usepackage{eqparbox}

\algnewcommand{\LeftComment}[1]{\Statex \(\triangleright\) #1}

\newcommand{\E}{\mathbb{E} \,} 
\newcommand{\R}{\mathbb{R} \,} 
\newcommand{\grad}{\nabla} 
\newcommand{\empE}{\hat{P}_\mathcal{S}} 
\newcommand{\stableTheta}{\Theta_{\lambda\text{-Stable}}}

\DeclareMathOperator{\BPTT}{BPTT} 

\newsavebox\CBox
\def\textBF#1{\sbox\CBox{#1}\resizebox{\wd\CBox}{\ht\CBox}{\textbf{#1}}}

\begin{document}

\title{Adaptively Truncating Backpropagation Through Time to Control Gradient Bias}

\author[1]{Christopher Aicher}
\author[2]{Nicholas J. Foti}
\author[1,2]{Emily B. Fox}
\affil[1]{Department of Statistics, University of Washington}
\affil[2]{Paul G. Allen School of Computer Science and Engineering, University of Washington}

\date{}
\maketitle

\begin{abstract}
Truncated backpropagation through time (TBPTT) is a popular method for learning in recurrent neural networks (RNNs) that saves computation and memory at the cost of bias by truncating backpropagation after a fixed number of lags. In practice, choosing the optimal truncation length is difficult: TBPTT will not converge if the truncation length is too small, or will converge slowly if it is too large. We propose an adaptive TBPTT scheme that converts the problem from choosing a temporal lag to one of choosing a tolerable amount of gradient bias. For many realistic RNNs, the TBPTT gradients decay geometrically \emph{in expectation} for large lags; under this condition, we can control the bias by varying the truncation length adaptively.  For RNNs with smooth activation functions, we prove that this bias controls the convergence rate of SGD with biased gradients for our non-convex loss. Using this theory, we develop a practical method for adaptively estimating the truncation length during training. We evaluate our adaptive TBPTT method on synthetic data and language modeling tasks and find that our adaptive TBPTT ameliorates the computational pitfalls of fixed TBPTT.
\vspace{0.2cm}
\end{abstract}

\section{Introduction}
\label{sec:intro}

Recurrent neural networks (RNNs) are a popular method of processing sequential data for wide range of tasks such as language modeling, machine translation and reinforcement learning.

As with most deep neural networks, RNNs are typically trained with gradient descent.
These gradients can be calculated efficiently using \emph{backpropagation through time} (BPTT) which applies backpropagation to the unrolled network~\cite{werbos1990backpropagation}.
For long sequential data, BPTT is both computationally and memory intensive, hence approximations based on \emph{truncating} BPTT (TBPTT) have been proposed~\cite{williams1995gradient,sutskever2013training}.
However, this truncation causes the gradients to be biased.
When the truncation level is not sufficiently large, the bias introduced can
cause SGD to not converge.
In practice, a large truncation size is chosen heuristically (e.g. larger than the expected `memory' of the system) or via cross-validation.

Quantifying the bias due to truncation is difficult.
Depending on the parameters of the RNN, the gradient bounds for backpropagation either \emph{explode} or \emph{vanish}~\cite{bengio1994learning,pascanu2013difficulty}.
When the gradients vanish, the bias in TBPTT can be bounded.
Recent work has analyzed conditions for the parameters of the RNN to enforce this vanishing gradient condition~\cite{miller2018stable}.
However, these approaches are very restrictive and prevent the RNN from learning long-term dependencies.

To bound the bias in TBPTT, instead of restricting the parameters,
we formalize the heuristic assumption that the gradients in backpropagation should rapidly decay for steps beyond the `memory' of the RNN.
Specifically, we assume gradient bounds that decay exponentially \emph{in expectation} rather than uniformly.
Under this assumption, we show that the bias in TBPTT decays geometrically and also how to estimate an upper bound for this bias given a minibatch of backpropagated gradients.
Using this estimated upper bound, we propose an adaptive truncation scheme to control the bias.
In addition, we prove non-asymptotic convergence rates for SGD when the \emph{relative} bias of our gradients is bounded.
In particular, we show that when the relative bias, $\delta < 1$, SGD with biased gradients converges at the rate $(1-\delta)^{-1}$ compared to SGD with exact (unbiased) gradients.
In our experiments on synthetic and text data we see that
(i) our heuristic assumption holds empirically for these tasks,
(ii) our adaptive TBPTT method controls the bias, while fixed TBPTT does not,
and (iii) that our adaptive TBPTT method is competitive with or outperforms the optimal fixed TBPTT.

The paper is organized as follows.
First, we review generic RNNs and BPTT in Section~\ref{sec:background}.
Then, we develop our theoretical results in Section~\ref{sec:theory}.
Using this theory, we develop estimators for the bias and propose an adaptive TBPTT SGD scheme in Section~\ref{sec:method}.
Finally, we test our proposed adaptive TBPTT training scheme in Section~\ref{sec:exp} on both synthetic and language modeling data.

\section{Background}
\label{sec:background}
A generic RNN with inputs $x_t \in \R^{d_x}$ and hidden states $h_t \in \R^{d_h}$ at time step $t$ evolves as
\begin{equation}
\label{eq:general_RNN}
h_t = H(h_{t-1}, x_t; \theta)
\end{equation}
for some function $H:\R^{d_h \times d_x} \rightarrow \R^{d_h}$ with parameters $\theta \in \Theta$ (e.g, weights and biases) of the model.
This framework encompasses most popular RNNs. We now present some examples that we use later.

\paragraph{Simple RNN:}
A simple RNN is a linear map composed with a non-linear activation function $\rho$
\begin{equation*}
h_t = H(h_{t-1}, x_t, \theta) = \rho(W h_{t-1} + U x_t)\enspace,
\end{equation*}
where the weight matrices $W, U$ are the parameters.

\paragraph{Long Short-Term Memory (LSTM):}
A popular class of sequence models are LSTMs~\cite{hochreiter1997long}.
The hidden state consists of a pair of vectors $h_t = (c_t, \tilde{h}_t)$
\begin{align*}
f_t &= \sigma(W_f \tilde{h}_{t-1} + U_f x_t)\\
i_t &= \sigma(W_i \tilde{h}_{t-1} + U_i x_t)\\
o_t &= \sigma(W_o \tilde{h}_{t-1} + U_o x_t)\\
z_t &= \tanh(W_z \tilde{h}_{t-1} + U_z x_t)\\
c_t &= i_t \circ z_t + f_t \circ c_{t-1} \\
\tilde{h}_t &= o_t \cdot \tanh(c_t) \enspace,
\end{align*}
where the eight matrices $W_*,U_*$, for $* \in \{f,i,o,z\}$ are the parameters, $\sigma$ is the logistic function, and $\circ$ denotes elementwise multiplication.
LSTMs are examples of \emph{gated-RNNs} which capture more complex time dependence through the use of gate variables $f_t, i_t, o_t$.

\paragraph{Stacked RNNs:}
RNNs can be composed by \emph{stacking} the hidden layers. This allows different layers to learn structure at varying resolutions.
Each RNN-layer treats the output of the previous layer as its input.
Specifically, each layer $l = 1,\ldots,N_l$ is described as
\begin{equation*}
h_t^{(l)} = H_{(l)}(h_{t-1}^{(l)}, h_{t}^{(l-1)}, \theta_{(l)})\enspace,
\end{equation*}
where $h_{t}^{(0)} = x_t$.

\subsection{Training RNNs}
To train an RNN, we minimize a loss that can be decomposed into individual time steps $\mathcal{L}(\theta) = \frac{1}{T} \sum_{t = 1}^T \mathcal{L}_t(\theta)$.
For example, the individual loss at step $t$ may measure the accuracy of $h_t$ at predicting target outputs $y_t \in \R^{d_y}$.

Given $X=x_{1:T}$ and $\mathcal{L}_{1:T}$, gradient descent methods are typically used to train the RNN to minimize the loss.
To scale gradient descent for large $T$, we use stochastic gradient descent (SGD), which uses a random estimator $\hat{g}$ for the full gradient $g = \grad_\theta \mathcal{L}$.

We first consider estimating $g$ using the gradient of the loss at a random individual time step, $\mathcal{L}_s$, where $s$ is a \emph{random} index drawn uniformly from $\{1,\ldots,T\}$.
In Section~\ref{sec:method:tbpttstyles}, we discuss efficient ways of computing Eq.~\eqref{eq:truncbackprop} for multiple losses $\mathcal{L}_{s:s+t}$ simultaneously.

Unrolling the RNN, the gradient of $\mathcal{L}_s$ is
\begin{equation}
\label{eq:single_backprop}
\grad_\theta \mathcal{L}_s = \frac{\partial \mathcal{L}_s}{\partial \theta} + \sum_{k = 0}^s \frac{\partial\mathcal{L}_s}{\partial h_{s-k}} \cdot \frac{\partial h_{s-k}}{\partial \theta} ,
\end{equation}
which can be efficiently computed using \emph{backpropagation through time} (BPTT)~\cite{werbos1990backpropagation, williams1995gradient}.

When $s$ is large, unrolling the RNN is both computationally and memory prohibitive;
therefore in practice, the backpropagated gradients in Eq.~\eqref{eq:single_backprop} are truncated after $K$ steps~\cite{williams1995gradient,sutskever2013training}
\begin{equation}
\label{eq:truncbackprop}
\widehat{\grad}_\theta \mathcal{L}_s^{K} = \frac{\partial \mathcal{L}_s}{\partial \theta} + \sum_{k = 0}^K \frac{\partial\mathcal{L}_s}{\partial h_{s-k}} \frac{\partial h_{s-k}}{\partial \theta},
\end{equation}
where $K \ll T$.

Let $\hat{g}_K =  \widehat{\grad}_\theta \mathcal{L}_s^{K}$ be our stochastic gradient estimator for $g$ truncated at $K$ steps.
When $K = T$, $\hat{g}_T$ is an unbiased estimate for $g$;
however in general for $K < T$, the gradient estimator $\hat{g}_K$ is \emph{biased}, $\E[\hat{g}_K(\theta)] \neq g(\theta)$. 
In practice, the truncation length $K$ is chosen heuristically to be ``large enough'' to capture the memory in the underlying process, in hope that the bias of $\hat{g}_K(\theta)$ does not affect the convergence of SGD.
In general, there are no guarantees on the size of this bias or the convergence of the overall optimization for fixed $K$.

\subsection{Vanishing and Exploding Bounds}
Let $\| \cdot \|$ denote the spectral norm for matrices and Euclidean norm for vectors.

To analyze the bias of $\hat{g}_K$, we are interested in the behavior of $\frac{\partial \mathcal{L}_t}{\partial h_{t-k}}$ for large $k$.
\citet{pascanu2013difficulty} observed
\begin{equation}
\label{eq:gradient_problem}
\frac{\partial \mathcal{L}_t}{\partial h_{t-k}} = \frac{\partial\mathcal{L}_t}{\partial h_t}\prod_{r = 1}^k \frac{\partial h_{t-r+1}}{\partial h_{t-r}} \enspace.
\end{equation}
In particular, the repeated product of Jacobian matrices $\frac{\partial h_t}{\partial h_{t-1}}$ cause Eq.~\eqref{eq:gradient_problem} to tend to \emph{explode} to infinity or \emph{vanish} to zero.
When $\theta$ has an exploding gradient, then the bias of $\hat{g}_{K}$ is unbounded.
When $\theta$ has a vanishing gradient, then the bias of $\hat{g}_{K}$ is small; however
if the gradient decays too rapidly, the RNN cannot learn long-term dependences~\cite{bengio1994learning,pascanu2013difficulty, miller2018stable}.
In practice, LSTMs and other gated-RNNs have been seen to work in a middle ground where (for appropriate $\theta$ and inputs $x_{1:T}$) the gate variables prevent the gradient from exploding or vanishing~\cite{hochreiter1997long, belletti2018factorized}.
However, \emph{gradient bounds}, based on the Jacobian $\|\frac{\partial h_t}{\partial h_{t-1}} \| \leq \lambda$, either explode or vanish
\begin{equation}
\label{eq:gradient_bound}
\left\|\frac{\partial \mathcal{L}_t}{\partial h_{t-k}}\right\| \leq \left\| \frac{\partial \mathcal{L}_t}{\partial h_t} \right\| \cdot \lambda^k \enspace.
\end{equation}
In light of Eq.~\eqref{eq:gradient_bound}, several approaches have been proposed in the literature to restrict $\theta$ to control $\lambda$. 

\emph{Unitary} training methods have been proposed to restrict $\theta$ such that $\lambda \approx 1$ for all $\theta$, but do not bound the bias of the resulting gradient~\cite{arjovsky2016unitary,jing2017tunable,vorontsov2017orthogonality}.

\emph{Stable} or \emph{Chaos-Free} training methods have been proposed to restrict $\theta$ such that
 $\lambda < 1$~\cite{laurent2016recurrent,miller2018stable}.
In particular, \citet{miller2018stable}
call an RNN $H$ \emph{stable} for parameters $\theta$ if it is a contraction in $h$, that is
\begin{equation}
\label{eq:lipschitz_restrict}
\sup_{\substack{h, h' \in \R^{d_h} \\ x \in \R^{d_x}}} \frac{\| H(h, x, \theta) - H(h', x, \theta)\|}{\| h - h' \|} \leq \lambda < 1
\end{equation}
and call an RNN $H$ \emph{data-dependent stable} if the supremum over Eq.~\eqref{eq:lipschitz_restrict} is restricted to \emph{observed} inputs $x \in X$.
Let $\stableTheta$ be the set of parameters $\theta$ satisfying Eq.~\eqref{eq:lipschitz_restrict}
and $\stableTheta^X$ be the set of parameters $\theta$ satisfying the data-dependent version.

 \citet{miller2018stable} show that if $\theta \in \stableTheta$ the RNN gradients have an exponential forgetting property (as $\| \frac{\partial H}{\partial h} \| < \lambda$),
 which prevents the RNN from learning long-term dependences.
 We desire conditions on $\theta$ where we can bound the bias, 
but are less restrictive than Eq.~\eqref{eq:lipschitz_restrict}.

\section{Theory}
\label{sec:theory}


In this section, we consider bounding the bias in TBPTT when $\theta$ satisfies a \emph{relaxation} of the contraction restriction Eq.~\eqref{eq:lipschitz_restrict}.
Under this condition and a bound on $\|\partial h_t/\partial \theta\|$, we show that both the \emph{absolute bias} and \emph{relative bias} are bounded and decay geometrically for large $K$.
Finally, we prove the convergence rate of SGD for gradients with bounded relative bias.
Full proofs of theorems can be found in the Supplement.

\subsection{Geometric Decay for Large Lags}
To reduce notation, we define $\phi_k = \| \frac{\partial{L}_s}{\partial h_{s-k}}\|$ to be the gradient norm of loss $\mathcal{L}_s$ at time $s$ with respect to the hidden state $k$ lags in the past.
Note that $\phi_k$ is a random variable as $s$ is a random index.

Our relaxation of Eq.~\eqref{eq:lipschitz_restrict} is to assume the norm of the backpropagated gradient $\phi_{k}$ decays geometrically, \emph{on average} for large enough lags $k$.
More formally,
\begin{assumption}
\label{assump:geometric_decay}
For $\theta$ fixed,
there exists $\beta \in (0,1)$ and $\tau \geq 0$ such that
\begin{equation}
\label{eq:geometric_decay_assump}
\E [\phi_{k+1}] \leq \beta \cdot \E [\phi_{k}] \enspace, \text{ for all } k \geq \tau
\end{equation}
\end{assumption}
This generalizes the vanishing gradient condition to hold \emph{in expectation}.


To contrast \ref{assump:geometric_decay} with $\theta \in \stableTheta$,
we observe that if $\theta \in \stableTheta$ then the gradient norms $\phi_k$ must \emph{uniformly} decay exponentially
\begin{equation}
\label{eq:stable_decay}
\phi_{k+1} \leq \lambda \cdot \phi_{k} \text{ for all } k \enspace.
\end{equation}
Eq.~\eqref{eq:geometric_decay_assump} is less restrictive than Eq.~\eqref{eq:stable_decay} as $\phi_{k+1} \leq \beta \cdot \phi_k$ only occurs for $k > \tau$ and in expectation rather than uniformly.
Denote the set of $\theta$ that satisfy \ref{assump:geometric_decay} with $\beta, \tau$ for inputs $X = x_{1:T}$ as $\Theta_{\beta, \tau}^X$.
Then, we have $\stableTheta \subseteq \stableTheta^X \subset \Theta_{\beta, \tau}^X$ for $(\beta,\tau) = (\lambda, 0)$.
Therefore \ref{assump:geometric_decay} is a more general condition.

For illustration, we present two examples where $\theta \in \Theta_{\beta, \tau}^X$ but $\theta \notin \stableTheta^X$.
\paragraph{Nilpotent Linear RNN:}
Consider a simple RNN with linear activation $h_t = W h_{t-1} + U x_t$ where $W$ is a \emph{nilpotent} matrix with index $k$, that is $W^k = 0$. Then $\partial h_t/\partial h_{t-k} = W^k = 0$ hence $(W,U) \in \Theta_{0, k}^X$; however the norm $\| \partial h_t/\partial h_{t-k} \| = \|W\|$ can be arbitrarily large, thus $(W,U) \notin \stableTheta$.
For this example, although $H$ is not stable, $k$-repeated composition $ h_t = (H \circ \cdots \circ H)(h_{t-k}, x_{t-k+1:t})$ is stable.

\paragraph{Unstable RNN with Resetting:}
Consider a generic RNN with $\theta$ chosen such that $H$ is unstable, Lipschitz with constant $\lambda > 1$, but with a \emph{resetting} property $H(h,x) = 0$, whenever $x \in \mathcal{X}_0$.
Then,
\begin{equation*}
\E_S[\phi_{k+1}] \leq  \Pr(x_s, \ldots, x_{s-k+1} \notin \mathcal{X}_0) \cdot \lambda \cdot \E_S[\phi_{k}]
\end{equation*}
Although the RNN is unstable, if the probability $\{x_s, \ldots, x_{s-k+1}\}$ \emph{hits} the resetting set $\mathcal{X}_0$ is greater than $1-\beta\lambda^{-1}$ for sufficiently large $k$ and for some $\beta < 1$,
then
$\Pr(x_s, \ldots, x_{s-k+1} \notin \mathcal{X}_0) \cdot \lambda \leq \beta$ and therefore $\theta \in \Theta_{\beta, k}$ with $\theta \notin \stableTheta^X$.
For this example, although $H$ is not stable, properties of input distribution $x_t$ can lead to $H$ have vanishing gradients in \emph{expectation}.

\subsection{TBPTT Bias Bounds}
\label{sec:tbpttbiasbounds}
We now show the usefulness of assumption \ref{assump:geometric_decay}, that is if $\theta \in \Theta_{\beta,\tau}^X$, then the bias of TBPTT is bounded and decays geometrically in $K$.
To do so, we additionally assume the partial derivatives of the hidden state with respect to the parameters is bounded.
\begin{assumption}
\label{assump:bound_hiddenstate_gradients}
For $\theta$ fixed, there exists $M < \infty$ such that
$\| \tfrac{\partial H(x_t, h_t, \theta)}{\partial \theta} \| \leq M$
for all $t$.
\end{assumption}
For most typical RNNs, where $\theta$ are weights and biases, if the inputs $x_t$ and $h_t$ are bounded then $M$ can be bounded and assumption \ref{assump:bound_hiddenstate_gradients} holds.

We now show both the bias of TBPTT is guaranteed to decay geometrically for large $K$.
\begin{theorem}[Bias Bound]
\label{thm:bias_bound}
If \ref{assump:geometric_decay} and \ref{assump:bound_hiddenstate_gradients} hold for $\theta$,
then the absolute bias is upper bounded as
\begin{equation*}
\| \E[\hat{g}_K(\theta)] - g(\theta) \| \leq \mathcal{E}(K,\theta) \enspace,
\end{equation*}
where
\begin{equation*}
\mathcal{E}(K,\theta) = \begin{cases}
M \cdot \E\left[ \sum_{k = K+1}^{\tau-1} \phi_{k} + \frac{\phi_{\tau}}{1-\beta}\right], & K < \tau \\
M \cdot \E[\phi_{\tau}] \cdot \frac{\beta^{K - \tau}}{1-\beta}, & K \geq \tau 
\end{cases} .
\end{equation*}
And the relative bias is upper bounded by
\begin{equation*}
\frac{\| \E[\hat{g}_K(\theta)] - g(\theta) \|}{\| g(\theta) \|} \leq \Delta(K,\theta) \enspace,
\end{equation*}
where
\begin{equation*}
\Delta(K,\theta) = \frac{\mathcal{E}(K,\theta)}{\max_{k \leq K} \| \E\hat{g}_k(\theta) \| - \mathcal{E}(k,\theta)} \enspace.
 \end{equation*}
when the denominator is positive.
\end{theorem}
Note that $\mathcal{E}(K, \theta)$ decays geometrically for $K \geq \tau$ and therefore $\Delta(K,\theta)$ decays geometrically for large enough $K$ (as the denominator is monotone increasing).

Using this upper bound, we define
$\kappa(\delta, \theta)$ to be the \textit{smallest truncation length} for the parameters $\theta$ with guaranteed relative bias less than $\delta$.
That is
\begin{equation}
\label{eq:Kdeltatheta_dfn}
\kappa(\delta, \theta) = \min_{K} \{ \Delta(K, \theta) < \delta\} \enspace.
\end{equation}
The geometric decay in $\Delta(K,\theta)$ ensures $\kappa(\delta, \theta)$ is small.

Finally, we define the \emph{adaptive} TBPTT gradient estimator to be $\hat{g}(\theta) = \hat{g}_{\kappa(\delta,\theta)}(\theta)$, which truncates BPTT after $\kappa(\delta, \theta)$ steps and therefore has relative bias less than $\delta$.


\subsection{SGD with Biased Gradients}
\label{sec:biased_sgd_bound}
We use stochastic gradient descent (SGD) to learn $\theta$
\begin{equation}
\label{eq:sgd}
\theta_{n+1} = \theta_n - \gamma_n \cdot \hat{g}(\theta_n) \enspace,
\end{equation}
where $\{\gamma_n\}_{n=1}^N$ are stepsizes.
When using SGD for non-convex optimization, such as in training RNNs,
we are interested in convergence to stationary points of $\mathcal{L}$,
where $\theta$ is called a $\epsilon$-stationary point of $\mathcal{L}(\theta)$ if $\| g(\theta)\|^2 \leq \epsilon$ \cite{nesterov2013introductory}.

Usually the stochastic gradients $\hat{g}$ are assumed to be unbiased;
however during training RNNs with TBPTT, the truncated gradients are biased.
Based on Section~\ref{sec:tbpttbiasbounds},
we consider the case when $\hat{g}(\theta)$ has a bounded relative bias,
\begin{equation}
\label{eq:bounded_rel_bias}
\| \E [\hat{g}(\theta)] - g(\theta) \| \leq \delta \|g(\theta)\|, \enspace \forall \theta \enspace.
\end{equation}
such as for our adaptive estimator $\hat{g}(\theta) = \hat{g}_{\kappa(\delta, \theta)}(\theta)$.

For gradients with bounded relative bias $\delta < 1$ (and the additional assumptions below), \citet{poljak1973pseudogradient} and \citet{bertsekas1989parallel} prove that the averaged SGD sequence $\theta_n$ \emph{asymptotically} converges to a stationary point when the stepsizes are $\gamma_n \propto n^{-1}$.
However, \emph{non-asymptotic} convergence rates are also of interest, as they are useful in practice to understand the non-asymptotic performance of the algorithm.
\citet{ghadimi2013stochastic} prove non-asymptotic convergence rates for SGD with unbiased gradients.
We extend these results to the case of SGD with biased gradients.
Similar results were previously investigated by \citet{chen2018stochastic}, but with weaker bounds\footnote{They consider the case of consistent but biased estimators, where the gradients are uniformly bounded and the relative error is controlled with high probability (rather than in expectation). See the Supplement for additional discussion.}.

For our SGD convergence bound we need two additional assumptions.
\begin{assumption}
\label{assump:smooth_loss} The gradients are $L$-Lipschitz
\begin{equation*}
\| g(\theta) - g(\theta') \| \leq L \| \theta - \theta' \|, \quad \forall \theta, \theta'\enspace.
\end{equation*}
\end{assumption}
This assumption holds for generic RNNs as long as the activation functions are smooth (e.g. $\sigma$ or $\tanh$, but not RELU).
Second, we assume that the variance of our stochastic gradient estimator is uniformly bounded.
\begin{assumption}
\label{assump:stoch_var_bound}
 $\E\| \hat{g}(\theta) - \E\hat{g}(\theta) \|^2 \leq \sigma^2$ for all $\theta$.
\end{assumption}
We can now present our main theorem regarding convergence rates of SGD with biased gradients.
\begin{theorem}[SGD with Biased Gradients]
\label{thm:bias_sgd}
If the relative bias of $\hat{g}(\theta)$ is bounded by $\delta < 1$ for all $\theta_n$ and \ref{assump:smooth_loss} and \ref{assump:stoch_var_bound} both hold,
then SGD, Eq.~\eqref{eq:sgd}, with stepsizes $\gamma_n \leq \frac{1-\delta}{L (1+\delta)^2}$ satisfies
\begin{equation}
\min_{n = 1,\ldots,N+1} \| g(\theta_{n}) \|^2 \leq \frac{2 D_\mathcal{L} + L \sigma^2 \sum_{n = 1}^N \gamma_n^2}{(1-\delta) \sum_{n = 1}^N \gamma_n} \enspace,
\end{equation}
where $D_\mathcal{L} = (\mathcal{L}(\theta_1) - \min_{\theta^*} \mathcal{L}(\theta^*))$.

In particular, the optimal fixed stepsize for $N$ fixed is
$\gamma_n = \sqrt{2D_\mathcal{L}/(NL\sigma^2)}$, for which the bound is
\begin{equation}
\min_{n = 1,\ldots,N+1} \| g(\theta_{n}) \|^2 \leq \frac{1}{1-\delta} \cdot \sqrt{\frac{8 D_\mathcal{L} L \sigma^2}{N}} \enspace.
\end{equation}
\end{theorem}
When $\delta = 0$, Thm.~\ref{thm:bias_sgd} reduces to the smooth non-convex convergence rate bound~\cite{ghadimi2013stochastic}.
The price of biased gradients in SGD is the factor $(1-\delta)^{-1}$.

In practice, when the constants in Thm.~\ref{thm:bias_sgd} are unknown (e.g. $D_\mathcal{L}$),
we can use a decaying stepsize $\gamma_n = \gamma \cdot n^{-1/2}$.
Once $\gamma_n \leq \frac{1-\delta}{L (1+\delta)^2}$, Thm.~\ref{thm:bias_sgd} implies 
\begin{equation}
\min_{n = 1, \ldots, N+1} \| g(\theta_n) \|^2  = \mathcal{O}\left(\frac{1}{1-\delta}\cdot\frac{\log n}{\sqrt{n}}\right) \enspace.
\end{equation}

Thm.~\ref{thm:bias_sgd} provides bounds of the form $\min_{n} \| g(\theta_n) \|^2 < \epsilon$, but does not say which iterate is best.
This can be accounted for by using a random-stopping time~\cite{ghadimi2013stochastic} or using variance reduction methods such as SVRG~\cite{reddi2016stochastic}.
We leave these extensions as future work. 
In our experiments, we select $\theta_n$ by evaluating performance on a validation set.

\paragraph{What happens when~\ref{assump:geometric_decay} is violated?}
We do not restrict $\theta$ to $\Theta_{\beta,\tau}^X$ during training,
therefore whenever  $\theta_n \notin \Theta_{\beta, \tau}$ (or in practice when $\kappa(\delta, \theta_n)$ is larger than our computational budget allows), we are not able to construct $\hat{g}(\theta_n)$ such that the relative bias is bounded.
In these cases, we use $\hat{g} = \hat{g}_{K_{\max}}$ for some large truncation $K_{\max}$.
Although $\hat{g}_{K_{\max}}$ does not satisfy \ref{assump:geometric_decay}, we assume the stationary points of interest $\theta^*$ are in $\Theta_{\beta, \tau}^X$ and that eventually $\theta_n$ ends in a neighborhood of a $\theta^*$ that is a subset of $\Theta_{\beta, \tau}$ where our theory holds.

The advantage of an adaptive $\kappa(\delta, \theta)$ over a fixed $K$ is that it ensures convergence when possible (relative bias $\delta < 1$), while being able to get away with using a smaller $K$ during optimization for computational speed-ups.

\section{Adaptive TBPTT}
\label{sec:method}
The theory in Sec.~\ref{sec:theory} naturally suggests an adaptive TBPTT algorithm that we summarize in Alg.~\ref{alg:adaptive_tbptt}.
Our method selects a truncation level by periodically estimating $\kappa(\delta, \theta)$ over the course of SGD.
In this section we describe our implementation of adaptive TBPTT and how to estimate the quantities necessary to choose the truncation level\footnote{Code for our implementation and experiments is available at \url{https://github.com/aicherc/adaptive_tbptt}.}.

\begin{algorithm}[h]
   \caption{Adaptive TBPTT}
   \label{alg:adaptive_tbptt}
\begin{algorithmic}[1]
  \STATE {\bfseries Input:} initial parameters $\theta_0$, initial truncation $K_0$, stepsizes $\gamma_{1:N}$, relative bias tolerance $\delta$, batch size $S$, window size $R$
  \FOR{$n = 0, \ldots, N-1$}
  \LeftComment{Compute adaptive truncation}
  \STATE Sample random minibatch $\mathcal{S}$ of size $S$
  \STATE Calculate $\phi_k$ using $\BPTT(R,1)$ for Eq.~\eqref{eq:gradnorm_minibatch}
  \STATE Calculate $\empE[\phi_{k}]$ for $k \in [0, R]$ using Eq.~\eqref{eq:empirical_expectation}
  \STATE Estimate $\hat{\beta}$ using Eq.~\eqref{eq:max_estimate_beta} or \eqref{eq:ols_estimate_beta}
  \STATE Set $K_n = \hat{\kappa}(\delta, \theta)$ using Eq.~\eqref{eq:23}
  \LeftComment{Update $\theta$ with streaming gradients}
  \FOR{$m = 1, \ldots, T/K_n$}
  \STATE Get minibatch $\mathcal{S}_m$ defined in Eq.~\eqref{eq:streaming_minibatch}
  \STATE Calculate $\hat{g}(\theta_n) =\BPTT(2K_n, K_n)$ on $\mathcal{S}_m$
  \STATE Set $\theta_n = \theta_n - \gamma_n \cdot \sqrt{K_n} \cdot \hat{g}(\theta_n)$
  \ENDFOR

  \STATE Set $\theta_{n+1} = \theta_n$
  \ENDFOR
  \STATE Return $\theta_{1:N}$.
\end{algorithmic}
\end{algorithm}

\subsection{Computing TBPTT Gradients}
\label{sec:method:tbpttstyles}
Following~\citet{williams1995gradient}, we denote $\mathrm{BPTT}(K_1,K_2)$ to be truncated backpropagation for $K_2$ losses $\mathcal{L}_{s-K_2+1:s}$  backpropagated over $K_1$ steps, that is
\begin{equation}
\label{eq:backprop}
\BPTT(K_1,K_2) =  \frac{1}{K_2} \sum_{k' = 0}^{K_2-1} \sum_{k = s-t}^{K_1} \frac{\partial \mathcal{L}_{s-k'}}{\partial h_{s-k}} \cdot \frac{\partial h_{s-k}}{\partial \theta}
\end{equation}
which can be computed efficiently using the recursion
\begin{equation}
b_k = \begin{cases}
b_{k-1} \cdot\dfrac{\partial h_{s-k+1}}{\partial h_{s-k}} + \dfrac{\partial \mathcal{L}_{s-k}}{\partial h_{s-k}}  & \text{ if } k < K_2 \\[12pt]
b_{k-1} \cdot\dfrac{\partial h_{s-k+1}}{\partial h_{s-k}} & \text{ if } k \geq K_2
\end{cases}
\end{equation}
with $\BPTT(K_1, K_2) = \frac{1}{K_2} \sum_{k = 0}^{K_1} b_k \cdot \frac{\partial h_{s-k}}{\partial \theta}$.
It is important to include the normalization factor $\frac{1}{K_2}$, to ensure regularizations (such as dropout or weight decay) do not change for different values of $K_2$.

When $K_1 = K$ and $K_2 = 1$, then we obtain $\BPTT(K,1) = \hat{g}_{K}$.
However, individually calculating $S$ samples of $\hat{g}_{K}$ using $\BPTT(K,1)$ takes $\mathcal{O}(JK)$ computation; whereas the same gradients (plus extra lags) can be computed using $\BPTT(J+K,K)$ in $\mathcal{O}(J+K)$ time.
In practice a popular default setting is to set $K_1 = K_2 = K$ (as done in TensorFlow~\cite{abadi2016tensorflow});
however this overweights small lags, as the $k$-th loss is only backpropagated only $K-k$ steps.
To ensure all $K$ losses are backpropagated at least $K$ steps, in our experiments we use $\BPTT(2K,K)$.

We also scale the gradient updates $\gamma_n \hat{g}$ by $\sqrt{K_n}$ to account for the decreasing variance in $\BPTT(2K, K)$ as $K$ increases.
If we did not scale the gradient updates, then as $K$ increases, the resulting increase in the computational cost per step is not offset.

To handle the initialization of $h_{s-k}$ in Eq.~\eqref{eq:backprop}, we partition $\{1,\ldots, T\}$ into $T/K$ contiguous subsequences 
\begin{equation}
\label{eq:streaming_minibatch}
\mathcal{S}_m = [(m-1)*K+1, \ldots ,m*K] \enspace.
\end{equation}
By sequentially processing $\mathcal{S}_m$ in order, the final hidden state of the RNN on $\mathcal{S}_m$ can be used as the input for the RNN on $\mathcal{S}_{m+1}$.



\subsection{Estimating the Geometric Decay Rate}
A prerequisite for determining our adaptive truncation level is estimating the geometric rate of decay in Eq.~\eqref{eq:geometric_decay_assump}, $\beta$, and the lag at which it is valid, $\tau$.
We consider the case where we are given a batch of gradient norms $\phi_k$
for $s$ in a random minibatch $\mathcal{S}$ of size $|\mathcal{S}| = S$ that are backpropagated over a window $k \in [0,R]$
\begin{equation}
\label{eq:gradnorm_minibatch}
\left\{\phi_{k} = \left\| \frac{\partial \mathcal{L}_s}{\partial h_{s-k}} \right\| \, : \, s \in \mathcal{S}, k \in [0,R] \right\} \enspace.
\end{equation}
The gradient norms $\phi_k$ can be computed iteratively in parallel using the same architecture as truncated backpropagation $\BPTT(R,1)$.
The window size $R$ should be set to some large value.
This should be larger than the $\tau$ of the optimal $\theta^*$.
The window size $R$ can be large, since we only estimate $\beta$ periodically.

We first focus on estimating $\beta$ given an estimate $\hat{\tau} > \tau$.
We observe that if $\hat{\tau} \geq \tau$ and \ref{assump:geometric_decay} holds,
then
\begin{equation}
\label{eq:bound_for_beta}
\log \beta \geq 
\max_{k' > k \geq \hat{\tau}}
\frac{\log \E[\phi_{k}] - \log \E[\phi_{k'}]}{k-k'} \enspace.
\end{equation}
Eq.~\eqref{eq:bound_for_beta} states that $\log \beta$ bounds the slope of $\log\E[\phi_{t}]$ between any pair of points larger than $\tau$.

Using Eq.~\eqref{eq:bound_for_beta}, we propose two methods for estimating $\beta$.
We replace the expectation $\E$ with the empirical approximation based on the minibatch $\mathcal{S}$
\begin{equation}
\label{eq:empirical_expectation}
\E[f(s)] \approx \empE[f(s)] := \frac{1}{\| \mathcal{S} \|} \sum_{s \in \mathcal{S}} f(s).
\end{equation}

Substituting the empirical approximation into Eq.~\eqref{eq:bound_for_beta} and restricting the points to $[\hat{\tau}, R]$, we obtain
\begin{equation}
\label{eq:max_estimate_beta}
\hat{\beta} = \log \left[
\max_{\hat{\tau} \leq k < k' \leq R}
\frac{\log \empE[\phi_{k}] - \log \empE[\phi_{k'}]}{k-k'}  \right]\enspace.
\end{equation}
Because this estimate of $\beta$ is based on the maximum it is sensitive to noise:
a single noisy pair of $\empE[\phi_{k}]$ completely determines $\hat\beta$ when using Eq.~\eqref{eq:max_estimate_beta}.
To reduce this sensitivity, we could use a $(1-\alpha)$ quantile instead of strict max; however, to account for the noise in $\empE[\phi_{k}]$, we use linear regression, which is a weighted-average of the pairs of slopes
\begin{equation}
\label{eq:ols_estimate_beta}
\tilde{\beta} = \log \left[
\frac{\sum_{k,k'} (\log \empE[\phi_{k}] - \log \empE[\phi_{k'}](k - k')]}{ \sum_{k,k'} (k-k')^2}  \right]\enspace.
\end{equation}
This estimator is not guaranteed to be consistent for an upper bound on $\beta$ (i.e. as $| \mathcal{S} | \rightarrow T$, $\tilde{\beta} \not\geq \beta$); however, we found Eq.~\eqref{eq:ols_estimate_beta} performed better in practice.

The correctness and efficiency of both methods depends on the size of both the minibatch $\mathcal{S}$ and the window $[\hat{\tau}, R]$. 
Larger minibatches improve the approximation accuracy of $\empE$.
Large windows $[\hat{\tau}, R]$ are necessary to check \ref{assump:geometric_decay}, but also lead to additional noise.

In practice, we set $\hat{\tau}$ to be a fraction of $R$;
in our experiments we did not see much variability in $\hat{\beta}$ once $\hat{\tau}$ was sufficiently large, therefore we use $\hat{\tau} =\frac{9}{10} R$.


\subsection{Estimating the Truncation Level}
To estimate $\kappa(\delta, \theta)$ in Eq.~\eqref{eq:Kdeltatheta_dfn}, we obtain empirical estimates for the absolute and relative biases of the gradient.

Given $\hat{\beta}, \hat{\tau}$,
our estimated bound for the absolute bias is
\begin{equation}
\hat{\mathcal{E}}(K, \theta) = \begin{cases}
\hat{M} \cdot \empE\left[ \sum_{k = K+1}^{\hat\tau-1} \phi_{k} + \frac{\phi_{S,\hat\tau}}{1-\hat\beta}\right], & K < \hat\tau \\
\hat{M} \cdot \empE[\phi_{\hat\tau}] \cdot \frac{\hat\beta^{K - \hat\tau}}{1-\hat\beta}, & K \geq \hat\tau
\end{cases} 
\end{equation}
where we estimate an upper-bound $\hat{M}$ for $M$ by keeping track of an upper-bound for
$h_{s,t}$ and $x_{s,t}$ during training. 

Similarly, our estimated bound for the relative bias is
\begin{equation}
\label{eq:est_relative_bias}
\hat{\Delta}(K, \theta) = \frac{\hat{\mathcal{E}}(K,\theta)}{\| \max_{k \leq K} \empE[\hat{g}_k(\theta)] \| - \hat{\mathcal{E}}(k,\theta)}
\end{equation}
In our implementation, we make the simplifying assumption that $ \|\empE[\hat{g}_K/\hat{M}]\| \approx \empE\|\sum_{k = 0}^K \phi_k\|$, which allows us to avoid calculating $\hat{M}$.

Our estimate for $K$ with relative error $\delta$ is
\begin{equation}
\label{eq:23}
\hat{\kappa}(\delta, \theta) = \min_{K\in [K_{\min}, K_{\max}]} \{ \hat{\Delta}(K, \theta) < \delta\} \enspace,
\end{equation}
where $K_{\min}$ and $K_{\max}$ are user-specified bounds.


\subsection{Runtime Analysis of Algorithm \ref{alg:adaptive_tbptt}}
Our adaptive TBPTT scheme, Algorithm~\ref{alg:adaptive_tbptt}, consists of estimating the truncation length (lines 3-7) and updating the parameters with SGD using TBPTT (lines 8-12); whereas a fixed TBPTT scheme skips lines 3-7.

Updating the parameters with $\BPTT(2K, K)$ streaming over $\{1, \ldots, T\}$ (lines 8-12), takes $\mathcal{O}(T/K \cdot K) = \mathcal{O}(T)$ time and memory\footnote{As $K$ increases, $\BPTT(2K,K)$ takes more time per step $\mathcal{O}(K)$, but there are less steps per epoch $\mathcal{O}(T/K)$; hence the overall computation time is $\mathcal{O}(T)$}.
The additional computational cost of adaptive TBPTT (lines 3-7) is dominated by the calculation of gradient norms $\phi_k$, Eq.~\eqref{eq:gradnorm_minibatch}, using $\BPTT(R,1)$, which takes $\mathcal{O}(R)$ time and memory.
If we update the truncation length $\alpha$ times each epoch, then the total cost for adaptive TBPTT is $\mathcal{O}(T+ \alpha R)$.
Therefore, the additional computation cost is negligible when $\alpha R << T$.

In Algorithm~\ref{alg:adaptive_tbptt}, we only update $K$ once per epoch ($\alpha = 1$); however more frequent updates ($\alpha < 1$) allow for less stale estimates at additional computational cost.


\section{Experiments}
\label{sec:exp}
In this section we demonstrate the advantages of our adaptive TBPTT scheme (Algorithm~\ref{alg:adaptive_tbptt}) in both synthetic copy and language modeling tasks.
For each task, we compare using fixed TBPTT and our adaptive TBPTT to train RNNs with SGD.
We evaluate performance using perplexity (PPL)
on the test set, for the $\theta_n$ that achieve the best PPL on the validation set.
To make a fair comparison, we measure PPL against the number of data passes (epochs) used in training (counting the data used to estimate $\hat{\kappa}(\delta, \theta_n)$).
We also evaluate the relative bias of our gradient estimates $\delta$ and truncation length $K$.
For our experiments with SGD, we use a fixed learning rate chosen to be the largest power of $10$ such that SGD did not quickly diverge.
In section~\ref{sec:geometric_decay}, we demonstrate that the best $\theta_n$ appear to satisfy \ref{assump:geometric_decay} (e.g., $\theta \in \Theta_{\beta,\lambda}$) by presenting the gradient norms $\E[\phi_k]$ against lag $k$.
In the Supplement we present additional experiments, applying our adaptive TBPTT scheme to temporal point process modeling~\cite{du2016recurrent}.

\subsection{Synthetic Copy Experiment}
\label{sec:exp-copy}

\begin{figure*}[tb]
\centering
\begin{minipage}[c]{.32\textwidth}
    \centering
    \includegraphics[width=\textwidth]{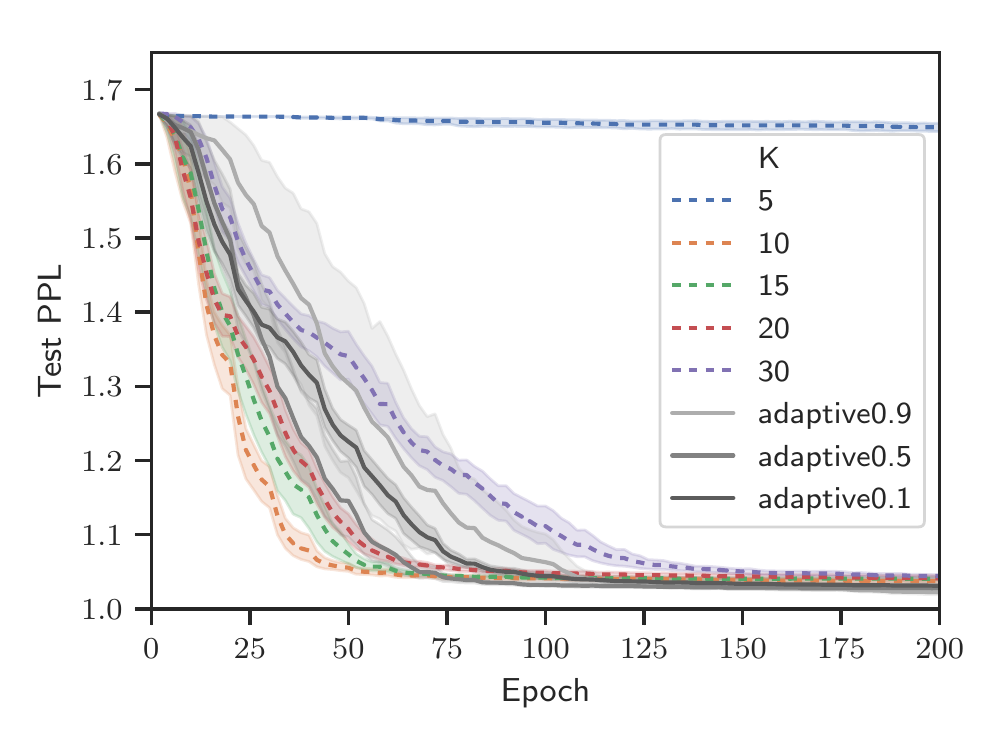}
\end{minipage}
\begin{minipage}[c]{.32\textwidth}
    \centering
    \includegraphics[width=\textwidth]{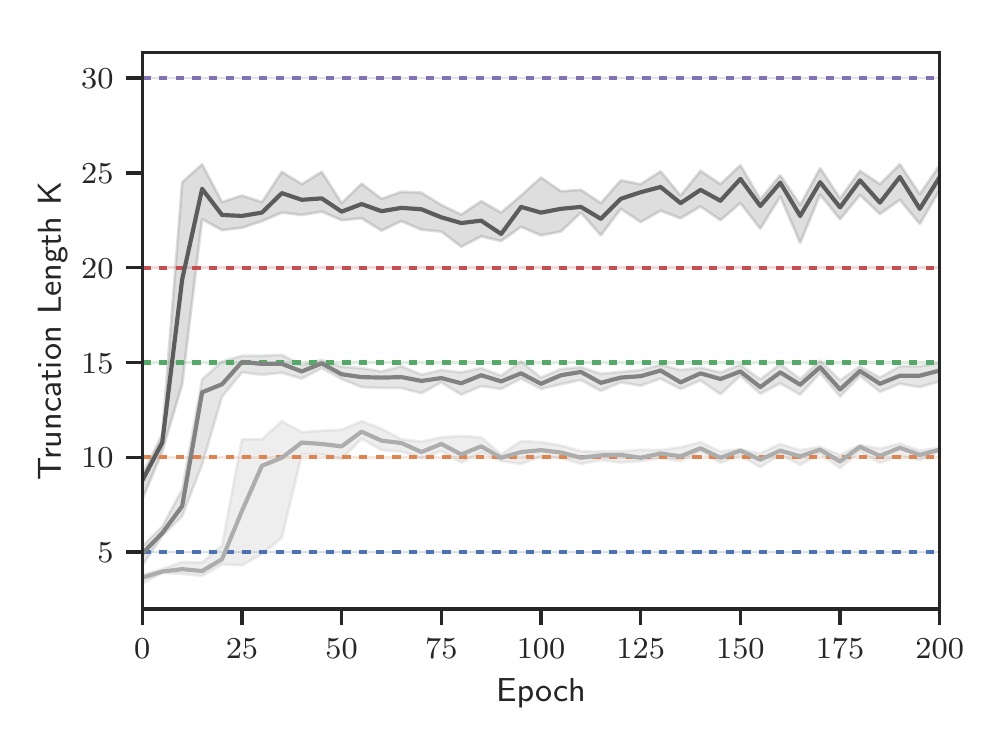}
\end{minipage}
\begin{minipage}[c]{.32\textwidth}
    \centering
    \includegraphics[width=\textwidth]{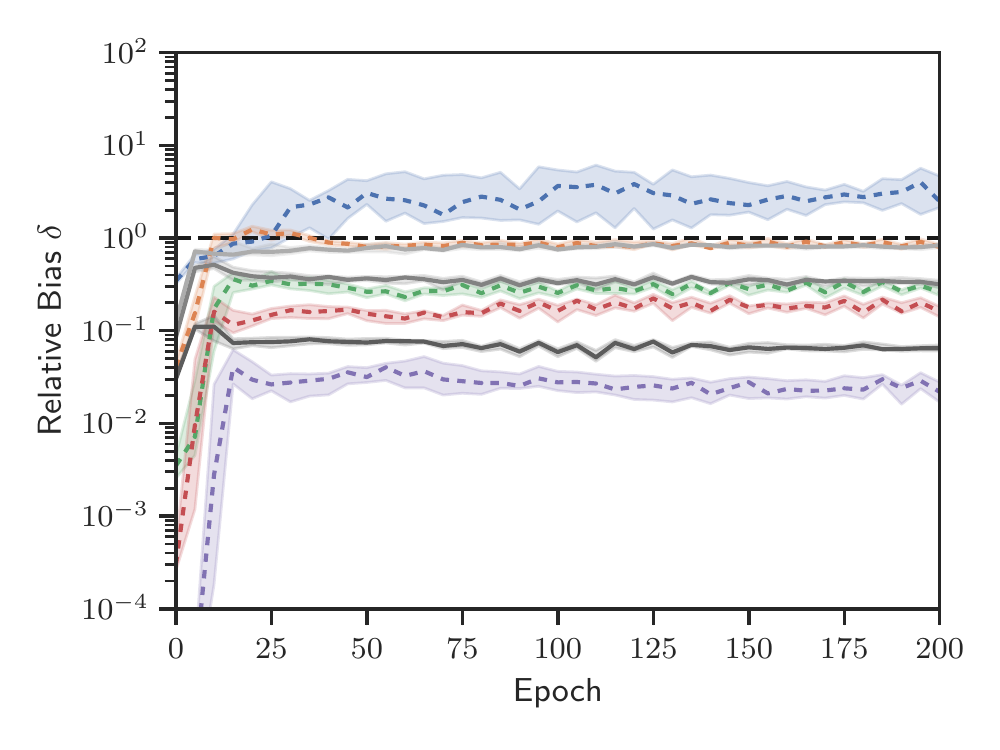}
\end{minipage}

\begin{minipage}[c]{.32\textwidth}
    \centering
    \includegraphics[width=\textwidth]{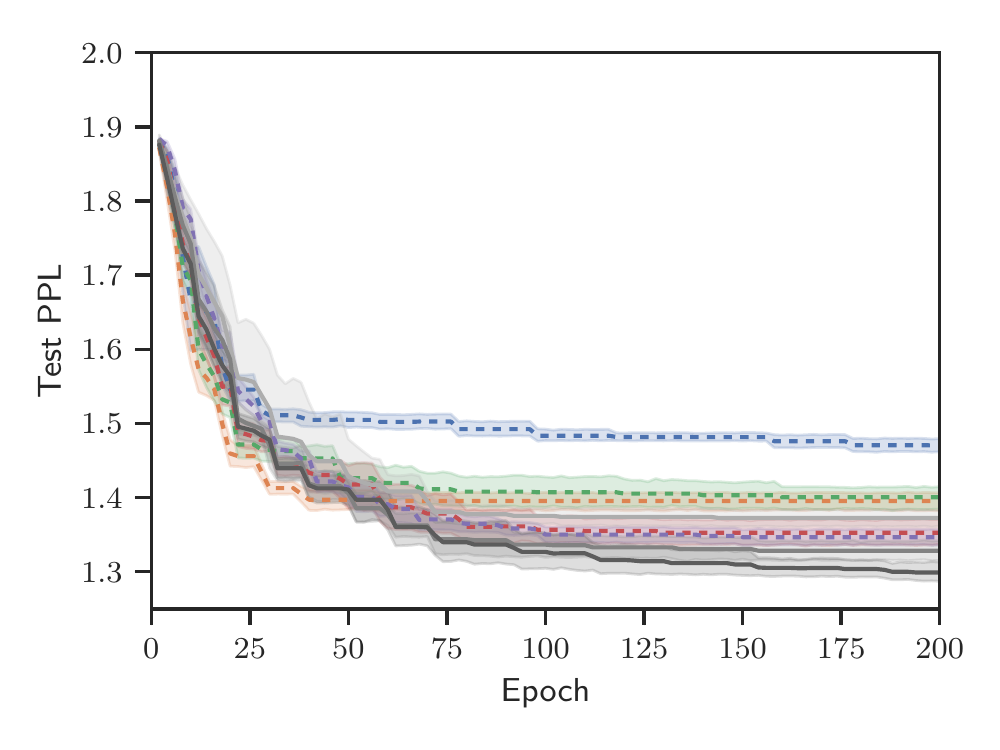}
\end{minipage}
\begin{minipage}[c]{.32\textwidth}
    \centering
    \includegraphics[width=\textwidth]{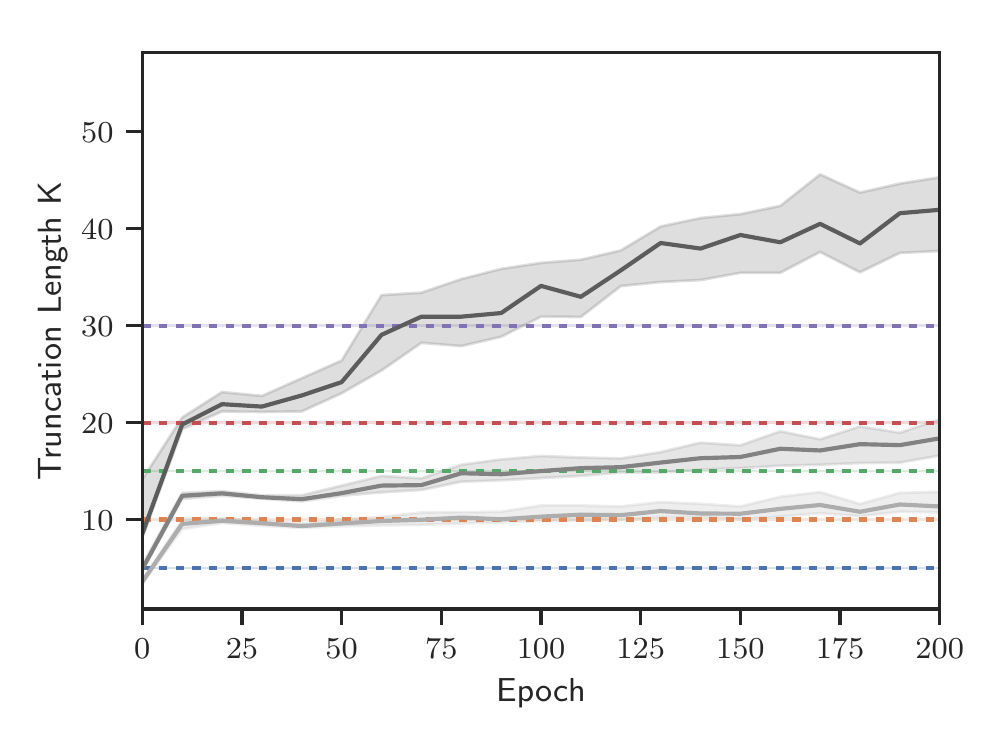}
\end{minipage}
\begin{minipage}[c]{.32\textwidth}
    \centering
    \includegraphics[width=\textwidth]{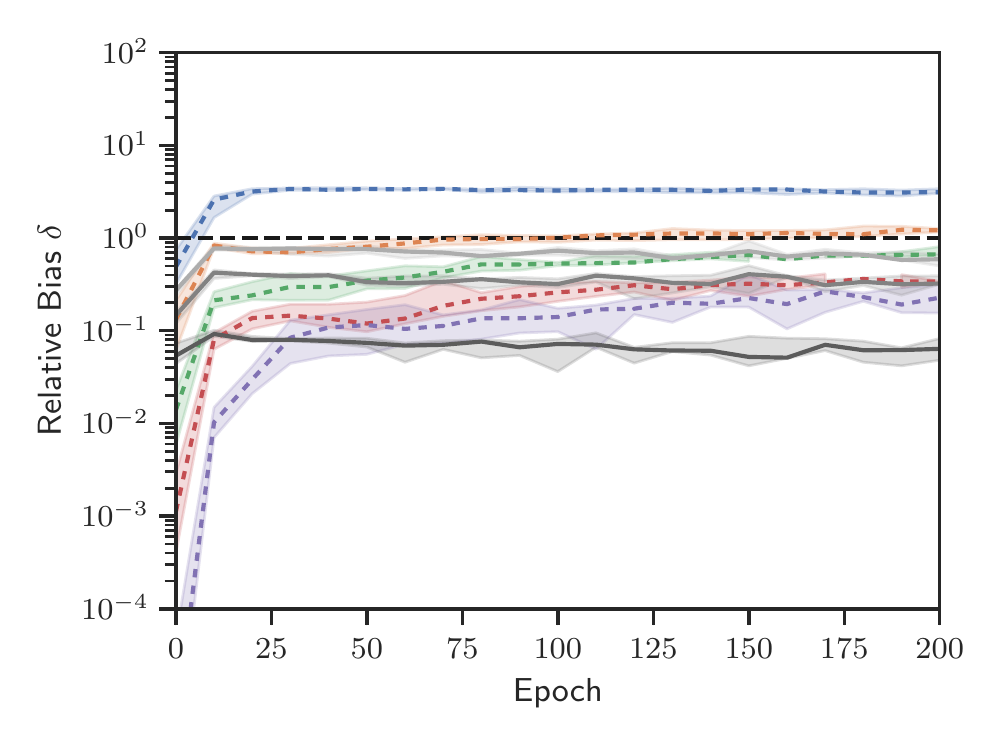}
\end{minipage}

\caption{Synthetic Copy Results: (left) Test PPL vs epoch, (center) $\hat{\kappa}(\delta, \theta_n)$ vs epoch (right) $\hat\delta(K)$ vs epoch. (Top) fixed $m = 10$, (bottom) variable $m \in [5, 10]$.
Error bars are 95 percentiles over 50 minibatch estimates.
Solid dark lines are our adaptive TBPTT methods, dashed colored lines are fixed TBPTT baselines. We see that the adaptive method converges in fewer epochs (left), while maintaining a controlled relative bias $\delta(K) \leq \delta$ (right).
}
\label{fig:copy}
\end{figure*}
The `copy' synthetic task is used to test the RNN's ability to remember information seen earlier~\cite{hochreiter1997long,arjovsky2016unitary,jing2017tunable, vorontsov2017orthogonality}.
We consider a special variant of this task from~\citep{mujika2018approximating}, which we review now.
Let $A = \{a_i\}$ be a set of symbols, where the first $I$ represent data and remaining two represent ``blank'' and ``start recall''.
Each input consists of sequence of $m$ random data symbols followed by
the ``start recall'' symbol and $m-1$ more blanks.
The desired output consists of $m$ blanks followed by the original sequence of $m$ data symbols.
For example when $m = 6$ and $A = \{A, B, C, -, \#\}$
\begin{verbatim}
       Input:   ACBBAB#-----
       Output:  ------ACBBAB
\end{verbatim}
We concatenate multiple of such inputs to construct $x_{1:T}$ and multiple outputs to construct $y_{1:T}$.
We expect that TBPTT with $K > m$ will perform well, while $K < m$ will perform poorly.

In our experiments, 
we consider both a \emph{fixed} $m = 10$ and a \emph{variable} $m$ drawn uniformly over $[5,10]$.
For the variable copy length experiment, we expect TBPTT to degrade more gradually as $K$ decreases.
We set $I = 6$ and use training data of length $T = 256,000$ and validation and test data of length $T=64,000$.

\paragraph{Model and Training Setup}
We train separate $2$-layer LSTMs with a embedding input layer and a linear output-layer to both the fixed- and variable- copy tasks by minimizing the cross-entropy loss.
The embedding dimension is set to $6$ and hidden and cell dimensions of the LSTM layers are set to $50$.
We train $\theta$ using SGD
using a batchsize of $S=64$ and a fixed learning rate of $\gamma = 1.0$
with fixed TBPTT $K \in [5, 10, 15, 20, 30]$ and our adaptive TBPTT method $\delta \in [0.9, 0.5, 0.1]$, $W = 100$, $K_0 = 15$ and $[K_{\min}, K_{\max}] = [2,100]$.
We set $W = 100$, $K_0 = 15$ and $[K_{\min}, K_{\max}] = [2,100]$ for Algorithm~\ref{alg:adaptive_tbptt}.

\paragraph{Results}
Figure~\ref{fig:copy} shows the results for the synthetic copy task.
The left figures present the test set PPL against the number of data epochs used in training.
We see that adaptive methods (black solid lines) perform as well as or better than the best fixed methods (colored dashed).
In particular, TBPTT with $K = 5$ (blue) does not learn how to accurately predict the outputs as $K$ is too small $5 = K \leq m = 10$.
On the other hand, $K = 30$ (purple) takes much longer to converge.
The center figures show how $\hat{\kappa}(\delta, \theta_n)$ evolves for the adaptive TBPTT methods over training.
The adaptive methods initially use small $K$ as the backpropagated gradient vanish rapidly in the early epochs;
however as the adaptive TBPTT methods learn $\theta$ the necessary $K$ for a relative error of $\delta$ increases until they eventually level off at $\kappa(\delta, \theta_N)$.
The right figures show the estimated relative bias $\delta$ of the gradient estimates during training.
We see that the adaptive methods are able to roughly control $\delta$ to be less than their target values,
while the fixed methods initially start with low $\delta$ and before increasing and leveling off.
Additional figures for the validation PPL and tables of numerical values can be found in the Supplement.

\subsection{Language Modeling Experiments}
\label{sec:exp-lm}
\begin{figure*}[h]
\centering
\begin{minipage}[c]{.32\textwidth}
    \centering
    \includegraphics[width=\textwidth]{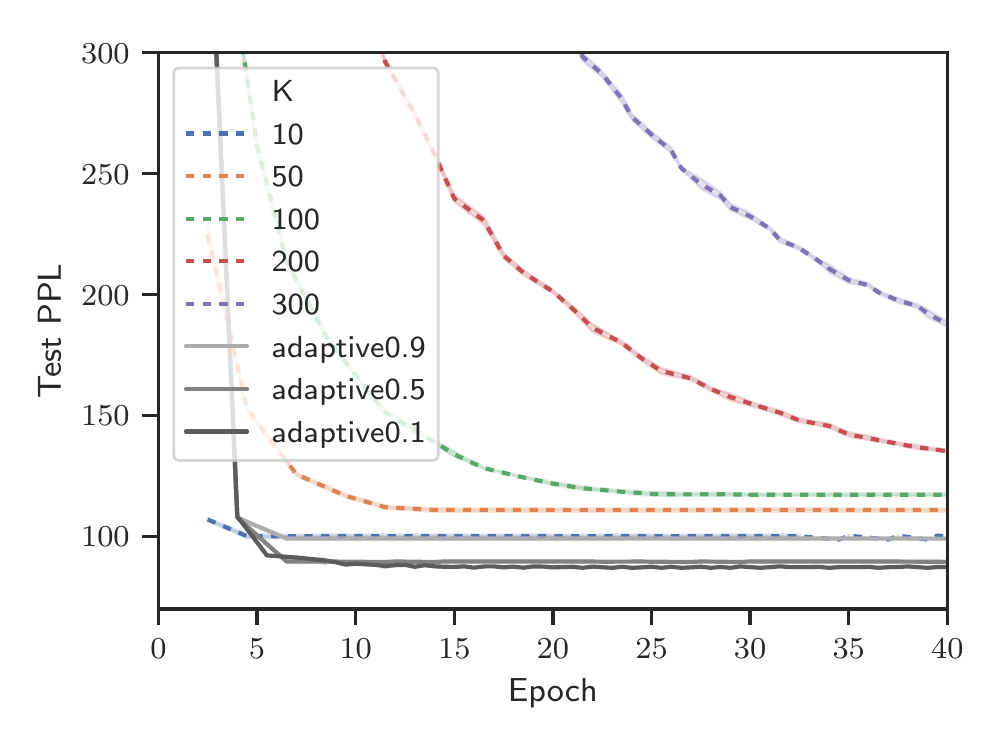}
\end{minipage}
\begin{minipage}[c]{.32\textwidth}
    \centering
    \includegraphics[width=\textwidth]{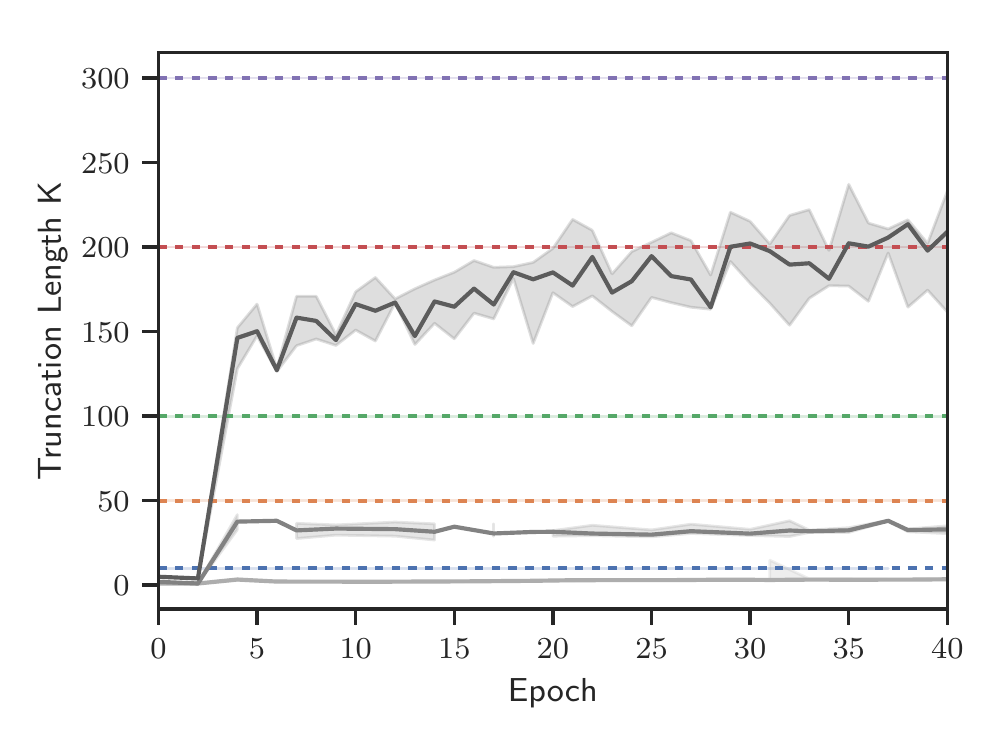}
\end{minipage}
\begin{minipage}[c]{.32\textwidth}
    \centering
    \includegraphics[width=\textwidth]{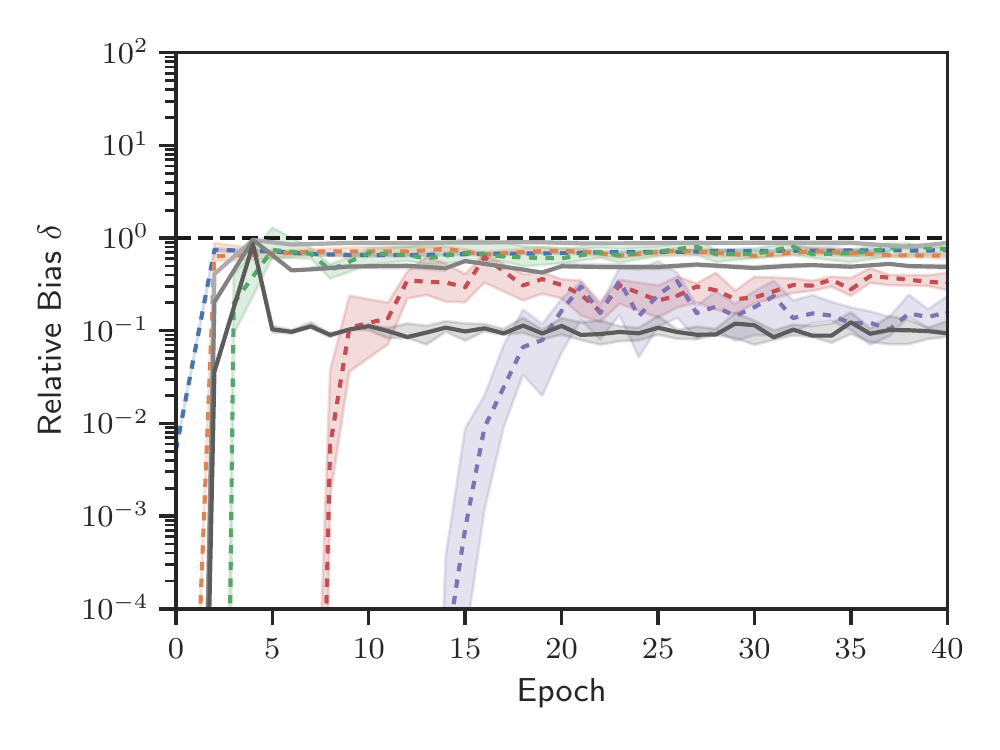}
\end{minipage}

\begin{minipage}[c]{.32\textwidth}
    \centering
    \includegraphics[width=\textwidth]{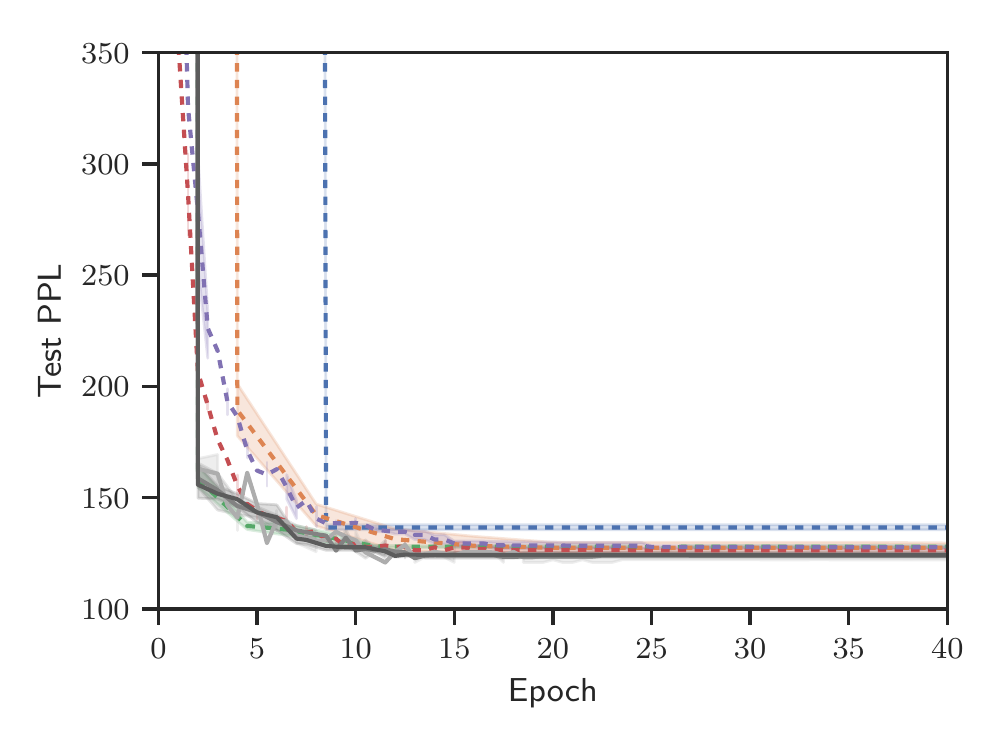}
\end{minipage}
\begin{minipage}[c]{.32\textwidth}
    \centering
    \includegraphics[width=\textwidth]{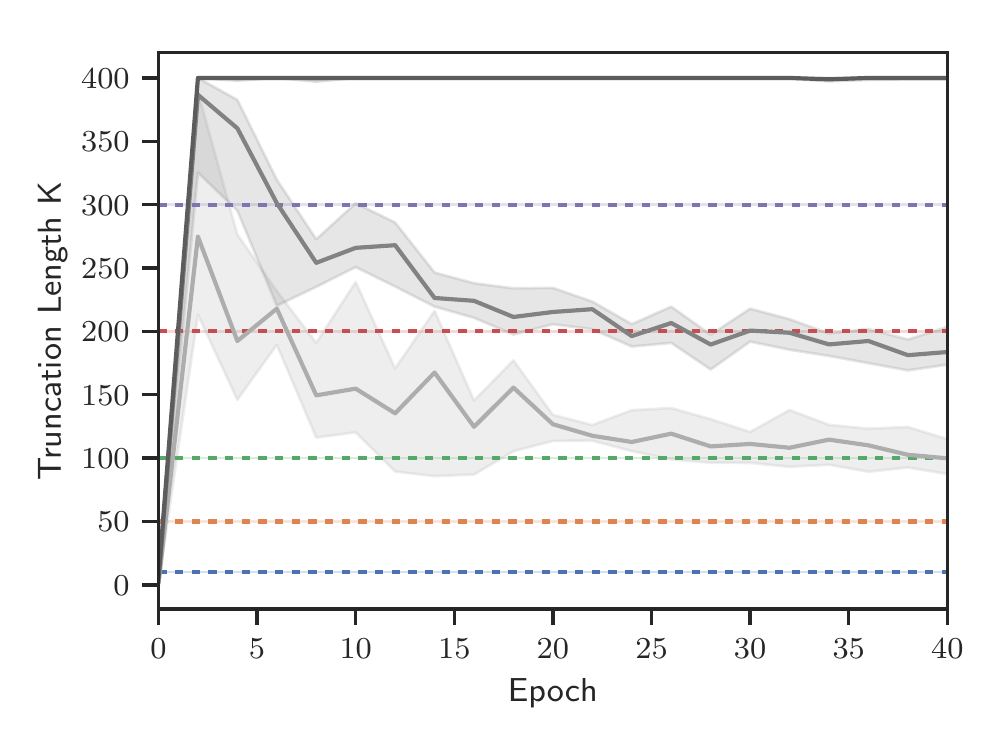}
\end{minipage}
\begin{minipage}[c]{.32\textwidth}
    \centering
    \includegraphics[width=\textwidth]{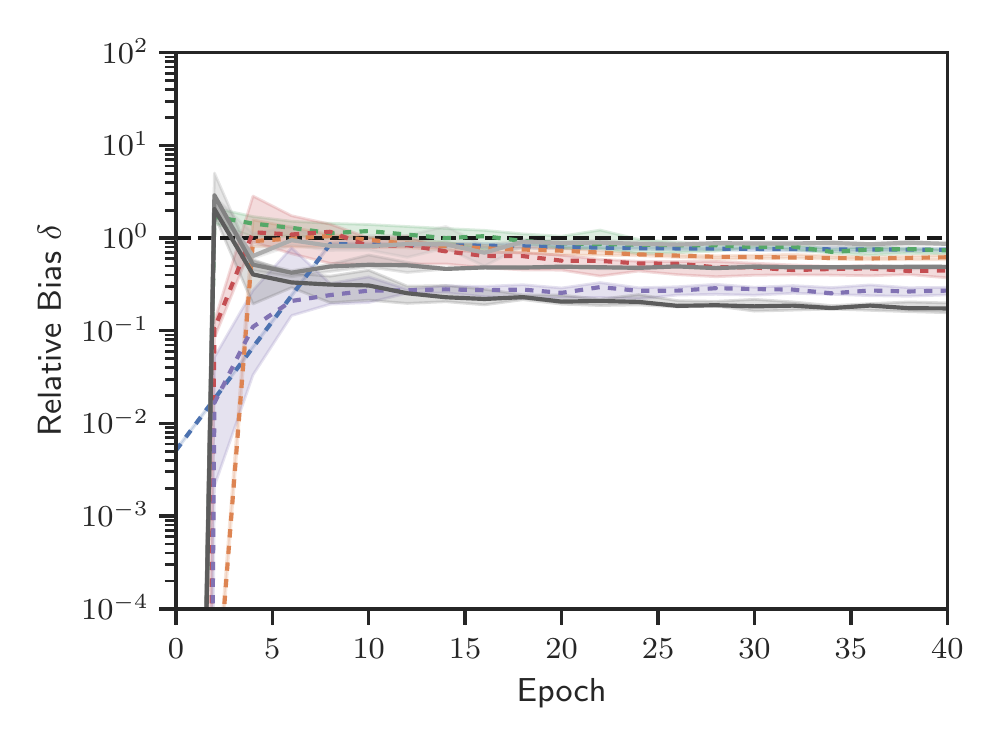}
\end{minipage}
\caption{Language Modeling Results. (left) Test PPL vs epoch, (center) $\hat{\kappa}(\delta, \theta_n)$ vs Epoch, (right) $\hat\delta(K)$ vs epoch. (Top) PTB (bottom) Wiki2.
Error bars are 95 percentiles over 50 minibatch estimates.
Solid dark lines are our adaptive TBPTT methods, dashed colored lines are fixed TBPTT baselines.
Our adaptive methods are competitive with the best fixed $K$ methods, while controlling the relative bias.
}
\label{fig:lm_results}
\end{figure*}
We also evaluate performance on language modeling tasks, where the goal is to predict the next word.
We train and evaluate models on both the Penn Treebank (PTB) corpus~\cite{marcus1993building, mikolov2010recurrent} and the Wikitext-2 (Wiki2) corpus~\cite{merity2016pointer}.
The PTB corpus contains about 1 millon tokens with a truncated vocabulary of 10k.
The Wikitext-2 is twice the size of PTB and with a vocabulary of 30k.

\paragraph{Model and Training Setup}
For both the PTB and Wiki2 corpus,
we train $1$-layer LSTMs with a word embedding layer input and a linear output layer.
The embedding dimension, hidden state, and cell state dimensions are all $900$ for the PTB following~\cite{lei2017training} and $512$ for the Wiki2 corpus following~\cite{miller2018stable}.
We use a batchsize of $S=32$ and a fixed learning rate of $\gamma = 10$
for fixed TBPTT $K \in [10, 50, 100, 200, 300]$ and our adaptive TBPTT method $\delta \in [0.9, 0.5, 0.1]$.
We set $W = 400$, $K_0 = 100$ and $[K_{\min}, K_{\max}] = [10,400]$ for Algorithm~\ref{alg:adaptive_tbptt}.

\paragraph{Results}
Figure~\ref{fig:lm_results} (left) presents the test PPL and $K_n$ against the training epoch for both language modeling tasks.
We again see that our adaptive methods are competitive with the best fixed $K$ methods, while controlling the relative bias.
From the $K(\delta,\theta_n)$ vs epoch figures, our adaptive method seems to quickly converge to a constant.
Therefore, on the real language data task, we have transformed the problem of selecting a fixed-$K$ to choosing a continuous parameter $\delta \in (0,1)$.
Additional figures for the validation PPL and tables of numerical values can be found in the Supplement.

\subsection{Empirically Checking \ref{assump:geometric_decay}}
\label{sec:geometric_decay}
\begin{figure}[tb]
\centering
\begin{minipage}[c]{.23\textwidth}
    \centering
    \includegraphics[width=\textwidth]{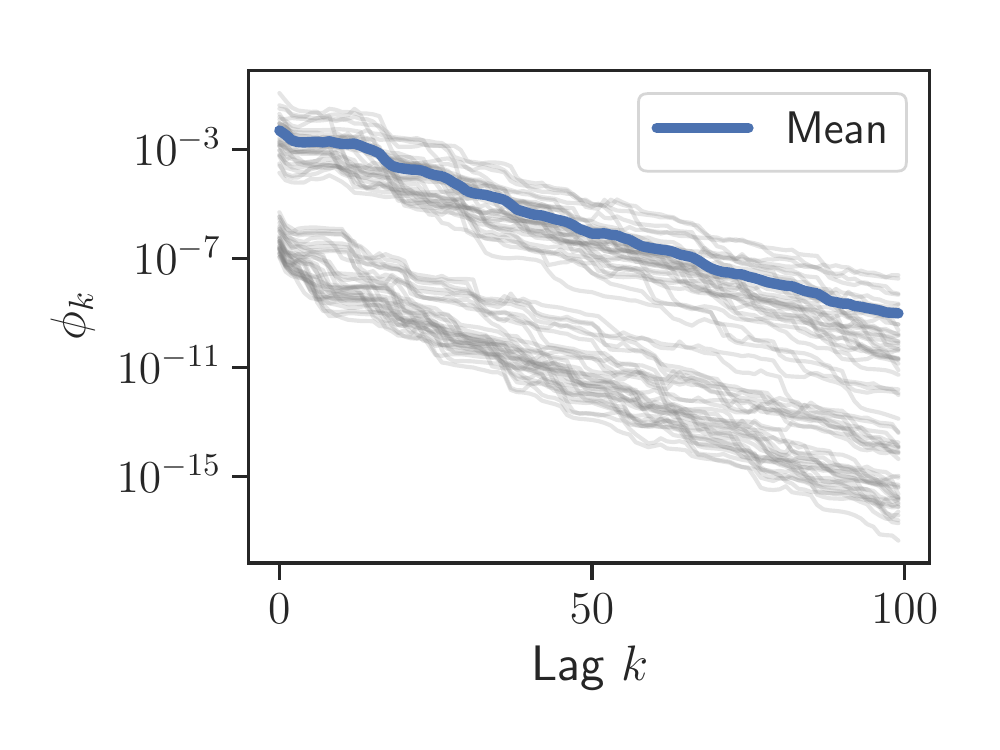}
\end{minipage}
\begin{minipage}[c]{.23\textwidth}
    \centering
    \includegraphics[width=\textwidth]{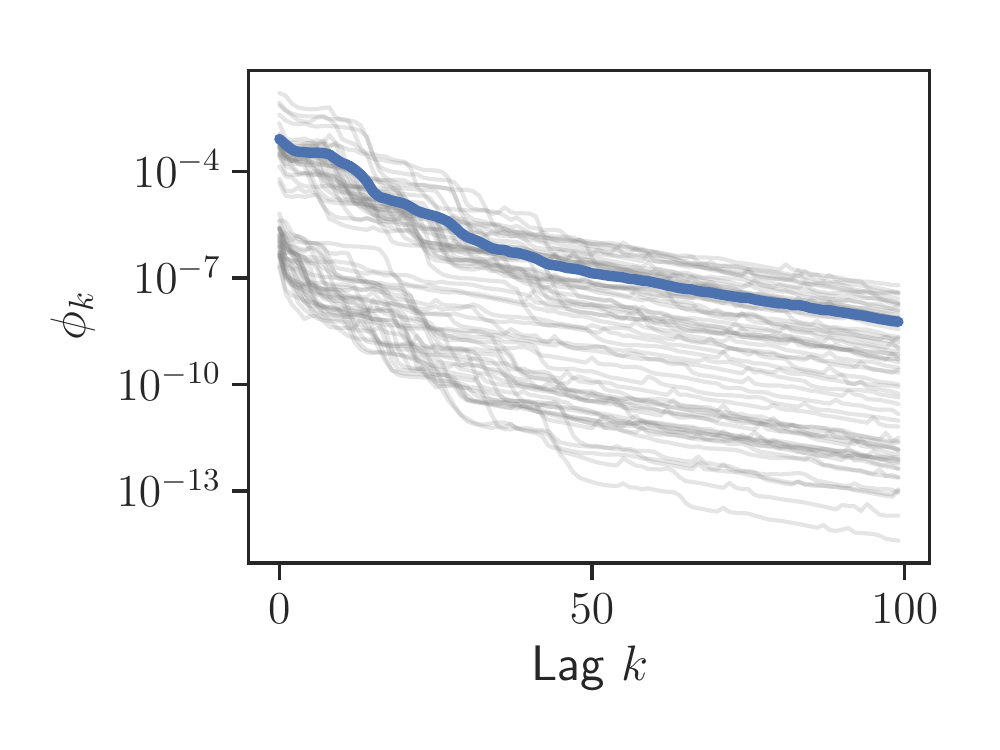}
\end{minipage}

\begin{minipage}[c]{.23\textwidth}
    \centering
    \includegraphics[width=\textwidth]{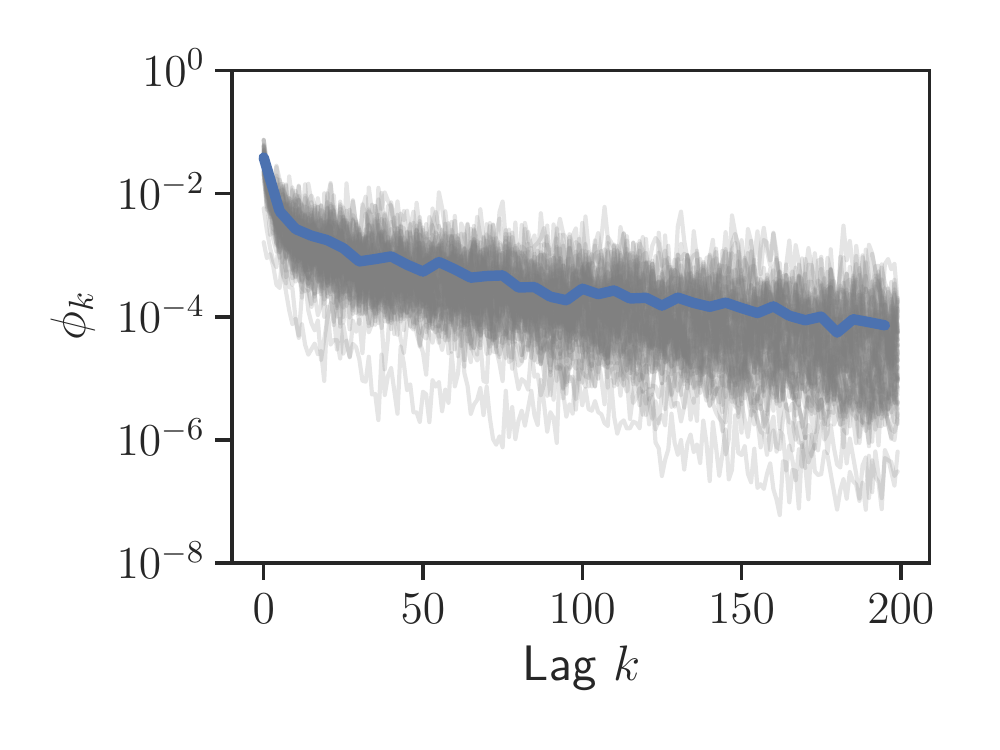}
\end{minipage}
\begin{minipage}[c]{.23\textwidth}
    \centering
    \includegraphics[width=\textwidth]{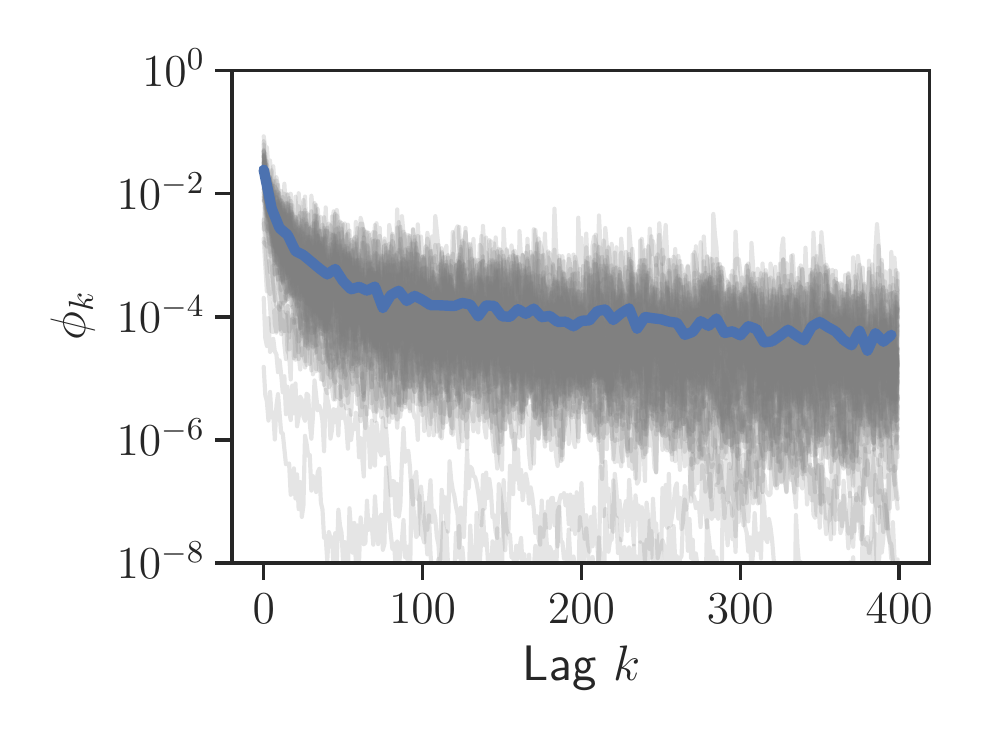}
\end{minipage}
\caption{Gradient Norms $\phi_{k}$ vs $k$ at the best $\theta$ showing geometric decay \emph{in expectation} (blue line) for large $k$. The gray lines are separate draws of $\phi_k$. (Top-left) fixed-length `copy' task, (top-right) variable-length `copy' task, (bottom-left) PTB, (bottom-right) Wiki2.}
\label{fig:grad_decay}
\end{figure}

Figure~\ref{fig:grad_decay} plots the gradient norm $\phi_k = \partial \mathcal{L}_s / \partial h_{s-k}$ vs $k$ evaluated at the best $\theta_n$ (as measured on the validation set).
Note that the $y$-axis is on a log-scale
We see that the expected norm $\empE[\phi_k]$ (blue-line) of the gradients decay geometrically for large $k$;
however any individual $\phi_{k}$ (gray lines) are quite noisy and do not strictly decay.
Therefore it appears that our RNNs satisfy \ref{assump:geometric_decay}, even though they are unstable, and thus the relative bias can be bounded.

\subsection{Challenges in higher dimensions}
During our experiments,
we found that when training RNNs with high-dimensional $h$, but without introducing regularization on $\theta$ (in the form of dropout or weight decay),
our estimates $\hat\beta$ were often close to or greater than $1$;
therefore our conservative relative error bound lead to extremely large (impractical) truncation estimates $K$.
During inspection, we found that although most dimensions of $\frac{\partial \mathcal{L}_s}{\partial h_{s-k}}$ decay rapidly with $k$,
a few dimensions did not and these dimensions cause the overall norm $\|\frac{\partial \mathcal{L}_s}{\partial h_{s-k}}\|$ to decay slowly, thus $\hat{\beta} \approx 1$.
However if these dimensions do not influence $\partial \mathcal{L}/\partial \theta$ (i.e. if $\frac{\partial h_t}{\partial \theta}$ is close to zero), then these dimensions should be ignored.
Therefore, to better apply our results to higher-dimensional $h$,
we suspect one should replace the Euclidean norm with a norm that weights dimensions of $\frac{\partial \mathcal{L}_s}{\partial h_{s-t}}$ by $\frac{\partial h_t}{\partial \theta}$ (such as the Mahalanobis norm $\| x \|_\Sigma = x^T\Sigma^{-1}x$ for some positive definite matrix $\Sigma$), but we leave this for future work.

\section{Discussion}
\label{sec:discussion}
In this work, we developed an adaptively truncating BPTT scheme for RNNs that satistify a generalized vanishing gradient property.
We show that if the gradient decays geometrically \emph{in expectation} \ref{assump:geometric_decay}, then we can control the relative bias of TBPTT (Theorem~\ref{thm:bias_bound}) and guarantee non-asymptotic convergence bounds for SGD (Theorem~\ref{thm:bias_sgd}).
We additionally show how to take advantage of these ideas in practice in Algorithm~\ref{alg:adaptive_tbptt}, by developing estimators for the relative bias based on backpropagated gradients.
We evaluate our proposed method on synthetic copy tasks and language modeling and find it performs similarly to the best fixed-$K$ TBPTT schemes, while still controlling the relative bias of the gradient estimates.
In future work, we are interested in methods that restrict the parameters to $\Theta_{\beta,\tau}$ and alternatives to the Euclidean norm for our error bounds in Section~\ref{sec:tbpttbiasbounds}.

\subsection*{Acknowledgements}
We thank the reviewers for their feedback.
We also thank members of the Dynamode lab at UW for their helpful discussions.
This work was supported in part by ONR Grant N00014-18-1-2862 and NSF CAREER Award IIS-1350133.

\bibliography{references.bib}   
\bibliographystyle{plainnat}

\clearpage
\appendix
\onecolumn
\renewcommand{\theequation}{\Alph{section}.\arabic{equation}}
\setcounter{section}{1}
\numberwithin{equation}{section}
\numberwithin{figure}{section}
\numberwithin{table}{section}

\begin{center}
\LARGE \textbf{Supplementary Material}
\end{center}

We start the supplement from `B' to avoid confusion between equation numbering and assumption numbering in the main text for  $(A-\#)$.

\section{Proofs for Section~\ref{sec:theory}}
\subsection{Proof of Theorem~\ref{thm:bias_bound}}
The proof of Theorem~\ref{thm:bias_bound} consists of two part.
First we bound the absolute bias by $\mathcal{E}(K,\theta)$ using assumptions~\ref{assump:geometric_decay} and~\ref{assump:bound_hiddenstate_gradients}.
Then we bound the relative bias using the triangle inequality
\begin{proof}[Proof of Theorem~\ref{thm:bias_bound}]
The bias of $\hat{g}_K$ is bounded by the expected error between $\hat{g}_K$ and $\hat{g}_T$
\begin{equation}
\label{eq:sm:1}
\| \E[\hat{g}_K(\theta)] - g(\theta) \| = 
\| \E[\hat{g}_K(\theta) - \hat{g}_T(\theta)] \| \leq 
\E \left[\| \hat{g}_K(\theta) - \hat{g}_T(\theta) \| \right] \enspace.
\end{equation}

Applying the triangle-inequality to the difference between $\hat{g}_K$ and $\hat{g}_T$ gives
\begin{equation}
\label{eq:sm:2}
\|\hat{g}_K(\theta) - \hat{g}_T(\theta)\| =
\left\|\sum_{k = K+1}^s \frac{\partial \mathcal{L}_s}{\partial h_{s-k}} \cdot \frac{\partial h_{s-k}}{\partial \theta} \right\|\leq 
\sum_{k = K+1}^s \left\|\frac{\partial \mathcal{L}_s}{\partial h_{s-k}}\right\| \cdot \left\| \frac{\partial h_{s-k}}{\partial \theta} \right\| \leq   
\sum_{k = K+1}^s \left\|\frac{\partial \mathcal{L}_s}{\partial h_{s-k}}\right\| \cdot M
\enspace,
\end{equation}
where in the last inequality we apply the assumption~\ref{assump:bound_hiddenstate_gradients}, $\|\partial h_t / \partial \theta\| < M$ for all $t$.
Taking the expectation with respect to $s$ of both sides of Eq.~\eqref{eq:sm:2} gives
\begin{equation}
\label{eq:sm:3}
\E \left[\| \hat{g}_K(\theta) - \hat{g}_T(\theta) \| \right] \leq
\sum_{k = K+1}^s \E\left\|\frac{\partial \mathcal{L}_s}{\partial h_{s-k}}\right\| \cdot M = 
\sum_{k = K+1}^s \E[\phi_{k}] \cdot M \enspace,
\end{equation}
where we recall that $\phi_{k} = \|\partial \mathcal{L}_s/\partial h_{s-k}\|$.

Recursively applying assumption~\ref{assump:bound_hiddenstate_gradients} to $\E[\phi_{\tau+t}]$ gives 
\begin{equation}
\label{eq:sm:4}
\E[\phi_{\tau+t}] \leq \beta \cdot \E[\phi_{\tau+t-1}] \leq \ldots \leq \beta^t \cdot \E[\phi_{\tau}] \enspace.
\end{equation}

Combining Eqs.~\eqref{eq:sm:1}, \eqref{eq:sm:3}, and \eqref{eq:sm:4} gives us the first half of the result
\begin{equation}
\| \E[\hat{g}_K(\theta)] - g(\theta) \| \leq 
\E \left[\| \hat{g}_K(\theta) - \hat{g}_T(\theta) \| \right] \leq
\sum_{k = K+1}^s \E_[\phi_{k}] \cdot M = \mathcal{E}(K, \theta) \enspace.
\end{equation}

To bound the relative error, we apply the reverse triangle inequality to $\| g(\theta)\|$
\begin{equation}
\| g(\theta) \| \geq \| \E[\hat{g}_K(\theta)] \| - \| \E[\hat{g}_K(\theta)] - g(\theta) \| \geq  \| \E[\hat{g}_K(\theta)] \| - \mathcal{E}(K, \theta)\enspace,
\end{equation}
when $\| \E[\hat{g}_K(\theta)] \| - \mathcal{E}(K, \theta) > 0$.

Since $\mathcal{E}(K,\theta)$ is an upper bound for the numerator and
$\| \E[\hat{g}_K(\theta)] \| - \mathcal{E}(K, \theta)$ is a lower bound for the denominator, we obtain the result
\begin{equation}
\frac{\| \E[\hat{g}_K(\theta)] - g(\theta) \|}{\|g(\theta)\|} \leq \frac{\mathcal{E}(K, \theta)}{\| \E[\hat{g}_K(\theta)] \| - \mathcal{E}(K, \theta)} = \delta(K,\theta)\enspace.
\end{equation}
\end{proof}

\subsection{Proof of Theorem~\ref{thm:bias_sgd}}
Let $\langle x_1, x_2 \rangle$ denote the inner-product between two vectors.

We first presents some Lemmas involving $\hat{g}(\theta)$ and $g(\theta)$
when the gradient has bounded relative bias $\delta$.
\begin{lemma}
\label{lemma:innerprods}
If $\hat{g}(\theta)$ has bounded relative bias of $\delta$ then 
\begin{equation}
\E\langle g(\theta), \hat{g}(\theta) - g(\theta) \rangle \leq \delta \| g(\theta) \|^2
\enspace \text{ and } \enspace
\E\langle g(\theta), \hat{g}(\theta) \rangle \geq (1-\delta) \| g(\theta) \|^2
\end{equation}
\end{lemma}
\begin{proof}[Proof of Lemma~\ref{lemma:innerprods}]
The first inequality follows from the Cauchy-Schwartz inequality and bound on relative bias
\begin{equation}
\E\langle g(\theta), \hat{g}(\theta) - g(\theta) \rangle = 
\langle g(\theta), \E[\hat{g}(\theta)] -g(\theta) \rangle \leq 
\| g(\theta) \| \| \E[\hat{g}(\theta) - g(\theta) \| \leq
\delta \| g(\theta) \|^2 \enspace.
\end{equation}
The second inequality follows immediately from the first
\begin{equation}
\langle g(\theta), \hat{g}(\theta) \langle = \langle g(\theta), g(\theta) \rangle + \langle g(\theta), \hat{g}(\theta) - g(\theta)\rangle \leq \|g(\theta)\|^2 - \delta \|g(\theta)\|^2 = (1-\delta) \|g(\theta)\|^2 \enspace.
\end{equation}
\end{proof}

The next lemma bounds the second moment of $\|\hat{g}(\theta)\|$.
\begin{lemma}
\label{lemma:bounded_second_moment}
If $\hat{g}$ has bounded relative bias $\delta$ and bounded variance $\sigma^2$ for all $\theta$ (assumption~\ref{assump:stoch_var_bound}), 
then
\begin{equation}
\E[\| \hat{g} \|^2] \leq (1+\delta)^2 \| g\|^2 + \sigma^2 \enspace. 
\end{equation}

\end{lemma}
\begin{proof}[Proof of Lemma~\ref{lemma:bounded_second_moment}]
\begin{equation}
\| \hat{g} \|^2 = \| g \|^2 + 2\langle g, \hat{g} - g \rangle + \| \hat{g} - g \|^2
\end{equation}
Take the expectation, we obtain the result
\begin{equation}
\E\| \hat{g} \|^2 = \| g \|^2 + 2\E\langle g, \hat{g} - g \rangle + \E\| \hat{g} - g \|^2 \enspace,
\end{equation}
where expand the mean-squared error into the bias squared plus variance
\begin{equation}
\E\| \hat{g} - g \|^2 = \|\E\hat{g} - g \|^2 + \E\| \hat{g} - \E\hat{g}\|^2 \leq \delta^2 \| g \|^2 = \sigma^2 \enspace.
\end{equation}
Therefore
\begin{equation}
\E\| \hat{g} \|^2 \leq \| g \|^2 + 2\delta \| g\|^2 + (\delta^2 \|g\|^2 + \sigma^2) = (1+\delta)^2 \| g \|^2 + \sigma^2 
\end{equation}
\end{proof}

We now begin the proof of Theorem~\ref{thm:bias_sgd} which builds off the proof in~\cite{ghadimi2013stochastic}.
\begin{proof}[Proof of Theorem~\ref{thm:bias_sgd}]
From the $L$-smoothness of $\mathcal{L}$, assumption~\ref{assump:smooth_loss}, we have
\begin{equation}
\mathcal{L}(\theta) - \mathcal{L}(\theta') - |\langle g(\theta), \theta - \theta'\rangle | \leq \frac{L}{2} \| \theta' - \theta \|^2, \quad \forall \theta, \theta' \enspace.
\end{equation}

Substituting $\theta = \theta_{n+1}$ and $\theta' = \theta_n$, where $\theta_{n+1}$ and $\theta_n$ are connected through SGD Eq.~\eqref{eq:sgd}, we obtain

\begin{align}
\mathcal{L}(\theta_{n+1}) 
&\leq
\mathcal{L}(\theta_n) + \langle g(\theta_n), \theta_{n+1} - \theta_n \rangle + \frac{L}{2} \| \theta_{n+1} - \theta_n \|^2 \\
&=
\mathcal{L}(\theta_n) - \gamma_n \langle g(\theta_n), \hat{g}(\theta_n) \rangle + \frac{L}{2} \gamma_n^2 \| \hat{g}(\theta_n) \|^2  \enspace.
\end{align}
Taking the expectation with respect to $\hat{g}(\theta_n)$ on both sides and using Lemmas~\ref{lemma:innerprods} and \ref{lemma:bounded_second_moment} gives us
\begin{align}
\E\mathcal{L}(\theta_{n+1}) 
&=
\mathcal{L}(\theta_n) - \gamma_n \E\langle g(\theta_n), \hat{g}(\theta_n) \rangle + \frac{L}{2} \gamma_n^2 \E\| \hat{g}(\theta_n) \|^2 \\
&\leq \mathcal{L}(\theta_n) - \gamma_n (1-\delta) \| g(\theta_n)\|^2 + \frac{L}{2} \gamma_n^2 ((1+\delta)^2 \| g(\theta_n) \|^2 + \sigma^2) \enspace.
\end{align}
Rearranging terms with $\gamma_n$ gives
\begin{equation}
\frac{\gamma_n (1-\delta)}{2} \left(2 - \gamma_n \frac{L(1+\delta)^2}{(1-\delta)}\right) \cdot \| g(\theta_n) \|^2 \leq \mathcal{L}(\theta_n) - \E\mathcal{L}(\theta_{n+1}) + \gamma_n^2 \frac{L \sigma^2}{2} \enspace.
\end{equation}
As we assume the stepsizes are $\gamma_n < \frac{1-\delta}{L(1+\delta)^2}$, therefore $(2 + \gamma_n \frac{L(1+\delta)^2}{(1-\delta)}) < 1$ and we can drop these terms.
Taking the summation over $n$ and taking the expectation with respect to $\hat{g}(\theta_n)$ for $n = 1,\ldots,N$ we obtain
\begin{equation}
\sum_{n = 1}^N \gamma_n \frac{(1-\delta)}{2}  \cdot \min_{n \in [1,N+1]} \| g(\theta_n) \|^2 \leq \mathcal{L}(\theta_1) - \E\mathcal{L}(\theta_{N+1}) + \sum_{n = 1}^N \gamma_n^2 \frac{L \sigma^2}{2} \enspace.
\end{equation}
Finally, we divide both sides by $\sum_n \gamma_n\frac{1-\delta}{2}$ and apply $\E\mathcal{L}(\theta_{N+1}) \geq \min_{\theta^*} \mathcal{L}(\theta^*)$ to obtain the result
\begin{equation}
\label{eq:sm:bound}
\min_{n \in [1,N+1]} \| g(\theta_n) \|^2 \leq \frac{ 2 D_{\mathcal{L}} +  L \sigma^2 \sum_{n = 1}^N \gamma_n^2}{ (1-\delta) \sum_{n = 1}^N \gamma_n} \enspace,
\end{equation}
where $D_{\mathcal{L}} = \mathcal{L}(\theta_1) - \min_{\theta^*} \mathcal{L}(\theta^*)$.

If we use a constant stepsize $\gamma_n = \gamma$ for all $n \in [1,N]$, then the optimal stepsize for $N$ steps of SGD is
\begin{equation}
\gamma = \sqrt{\frac{2 D_{\mathcal{L}}}{NL\sigma^2}}  \quad \text{ which achieves } \enspace\min_{n \in [1,N+1]} \| g(\theta_n) \|^2 \leq \frac{1}{1-\delta} \cdot \sqrt{\frac{8 D_\mathcal{L} L \sigma^2}{N}} \enspace.
\end{equation}

If instead a decaying $\mathcal{O}(n^{-1/2})$ stepsize is used, then
the numerator of Eq.~\eqref{eq:sm:bound} grows as a harmonic series $\mathcal{O}(\sum_n n^{-1}) = \mathcal{O}(\log n)$, while the denominator grows $\mathcal{O}(\sum_n n^{-1/2}) = \mathcal{O}(n^{1/2})$.
Therefore the overall rate is $\mathcal{O}(n^{-1/2} \log n)$.
\end{proof}

\subsection{Comparison of Bounds to~\citep{chen2018stochastic}}

In Section~\ref{sec:biased_sgd_bound} for Theorem~\ref{thm:bias_sgd}, we assume the \emph{relative bias} is bounded, that is $\| \E[\hat{g}(\theta)] - g(\theta) \| \leq \delta \| g(\theta) \|$ for all $\theta$ (Eq.~\eqref{eq:bounded_rel_bias}).
\citet{chen2018stochastic} prove similar results to Theorem~\ref{thm:bias_sgd}, where they assume the relative error of each gradient is bounded in high probability, that is
there exists $\delta, \epsilon > 0$ such that
\begin{equation}
\label{eq:high_prob_bound}
\Pr( \| \hat{g}(\theta) - g(\theta) \| \leq \delta \| g(\theta) \|) > 1-\epsilon \enspace, \text{ for all } \theta \enspace.
\end{equation}
Although
Markov's inequality implies that if the relative bias is bounded by $\delta \cdot \epsilon$, when Eq.~\eqref{eq:high_prob_bound} holds for $\delta, \epsilon$, their non-convex optimization results only hold in high probability rather than uniformly.
A key drawback of their results, is that the relative error must be bounded in high probability for all steps of SGD $(\hat{g}_{1:N})$;
therefore the required $\epsilon$ for each step depends on the total number of SGD steps during training~\citep[see][Eq.(7) and Theorem 5]{chen2018stochastic}.
Specifically, \citet{chen2018stochastic} observe that the probability the relative error is controlled for all $N$ steps is bounded by $1-\epsilon_{\text{total}} \leq (1-\epsilon)^N$ under the additional assumption that the noise in $\hat{g}(\theta)$ is independent.
For their results to hold with probability $1 -\epsilon_{\text{total}}$ after $N$ steps, each gradient must have a relative error bound with $\epsilon \leq 1 - (1-\epsilon_{\text{total}})^{1/N}$.
\citet{chen2018stochastic} achieve this by restricting $\epsilon \leq \epsilon_{\text{total}}/N$.
Our result assumes the relative error is bounded in expectation, which sides steps this issue.
However our results are not as robust in the sense that they do not hold if the noise in $\hat{g}(\theta)$ does not have an expected value (e.g. if $\hat{g}(\theta) - g(\theta)$ is Cauchy).

\section{Additional Experiments}
This section provides additional tables and figures for the experiments in Section~\ref{sec:exp} as well as results on time series prediction with temporal point processes.

In our experiments, we selected the stepsize $\gamma$ for SGD by performing a grid search over powers of 10 and selected the largest stepsize that did not diverge for fixed TBPTT (with $K = 15$ for the synthetic tasks, $K=100$ for the language modeling tasks, and $K=6$ for the temporal point process tasks).
We also consider adaptive and decaying stepsizes (specifically ADADELTA, SGD with Momentum, and epoch-wise stepsize decay); however, we did not see a significant difference in results.

\subsection{Synthetic `Copy' Experiment}
Figure~\ref{fig:supp_copy} shows the validation PPL for the two experiments in Section~\ref{sec:exp-copy}.
The left pair of figures show the validation PPL while the right pair shows the cumulative minimum (i.e. the `best') validation PPL.
The test PPL plots in Figure~\ref{fig:copy} are piecewise constant evaluated using these `best' validation PPL parameters.
The top row corresponds to the fixed-memory $m=10$ copy experiment, and we see the loss decays relatively smoothly.
The bottom row corresponds to the variable-memory $m\in[5,10]$ copy experiment, and we see heavy oscillation in the validation error as it decays.
Table~\ref{tab:copy} is a table of the test PPL results evaluated at the `best' validation PPL.
This table provides the numeric values of the `best' PPL values for Figures~\ref{fig:copy} and \ref{fig:supp_copy}.
We see that the adaptive TBPTT perform as well as or outperform the best fixed $K$ TBPTT.

\begin{figure}[H]
\centering
\begin{minipage}[c]{.32\textwidth}
    \centering
    \includegraphics[width=\textwidth]{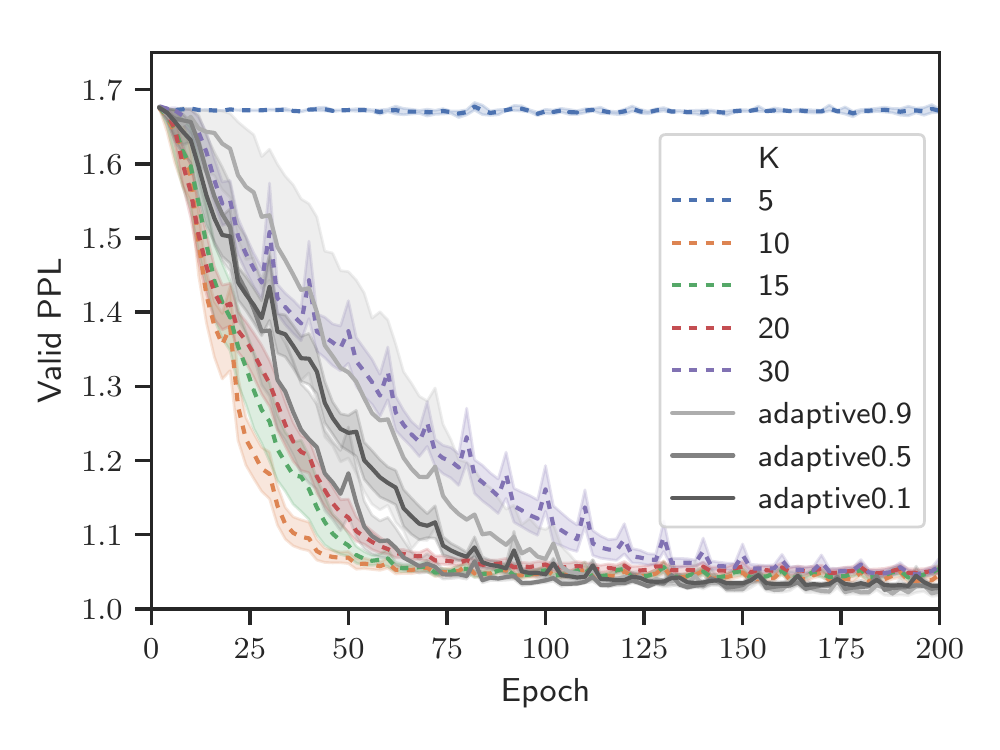}
\end{minipage}
\hspace{1em}
\begin{minipage}[c]{.32\textwidth}
    \centering
    \includegraphics[width=\textwidth]{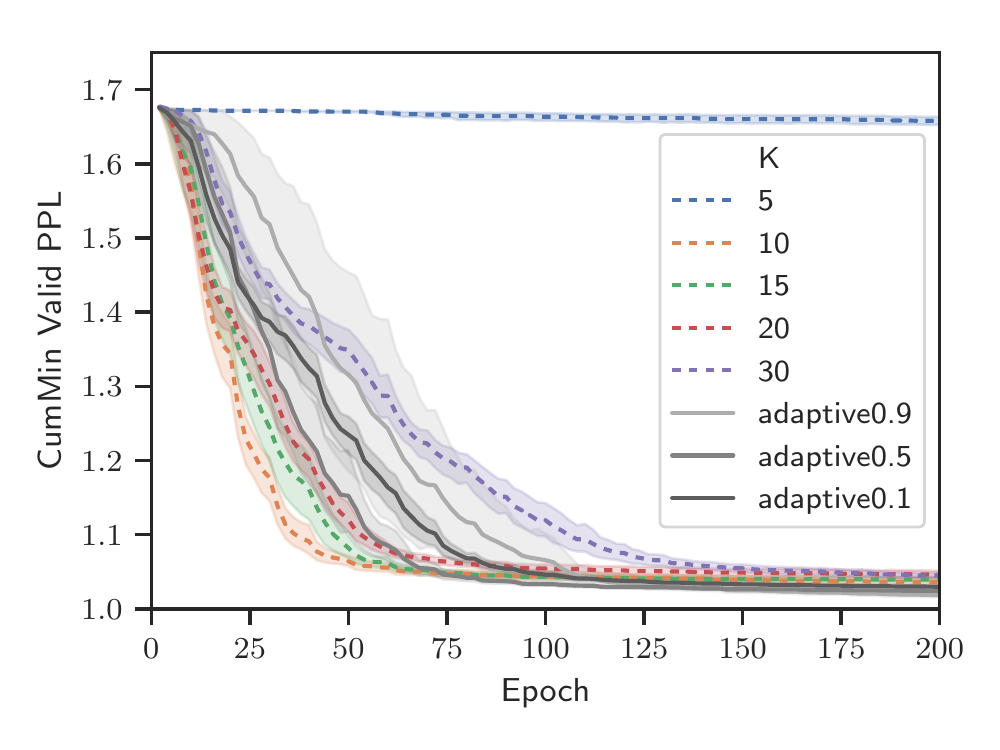}
\end{minipage}

\begin{minipage}[c]{.32\textwidth}
    \centering
    \includegraphics[width=\textwidth]{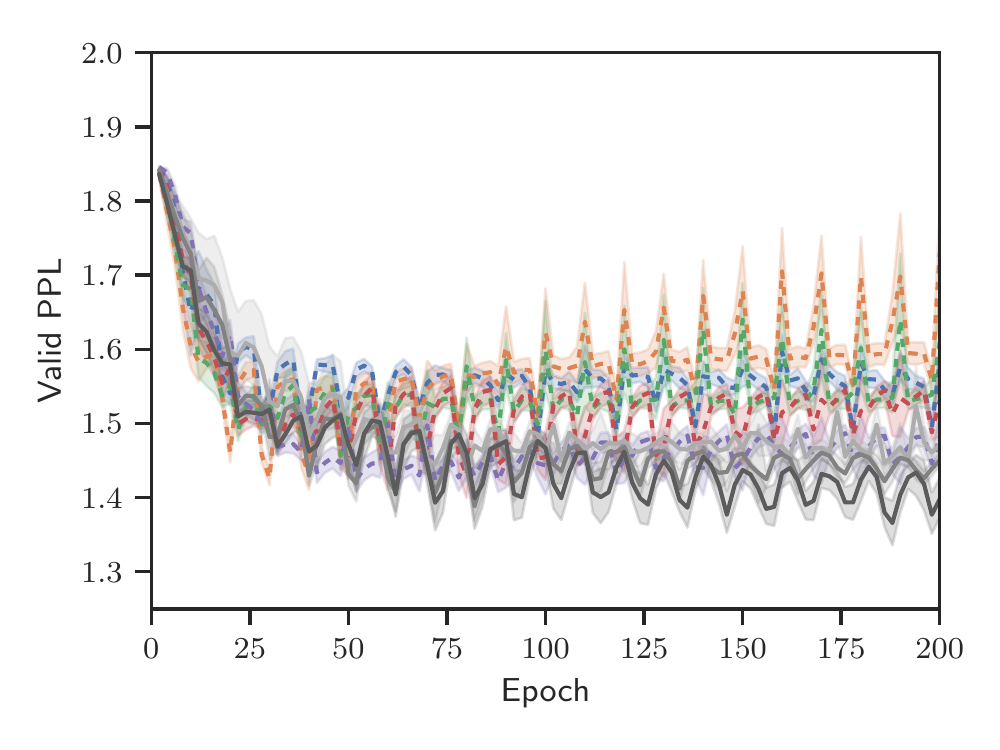}
\end{minipage}
\hspace{1em}
\begin{minipage}[c]{.32\textwidth}
    \centering
    \includegraphics[width=\textwidth]{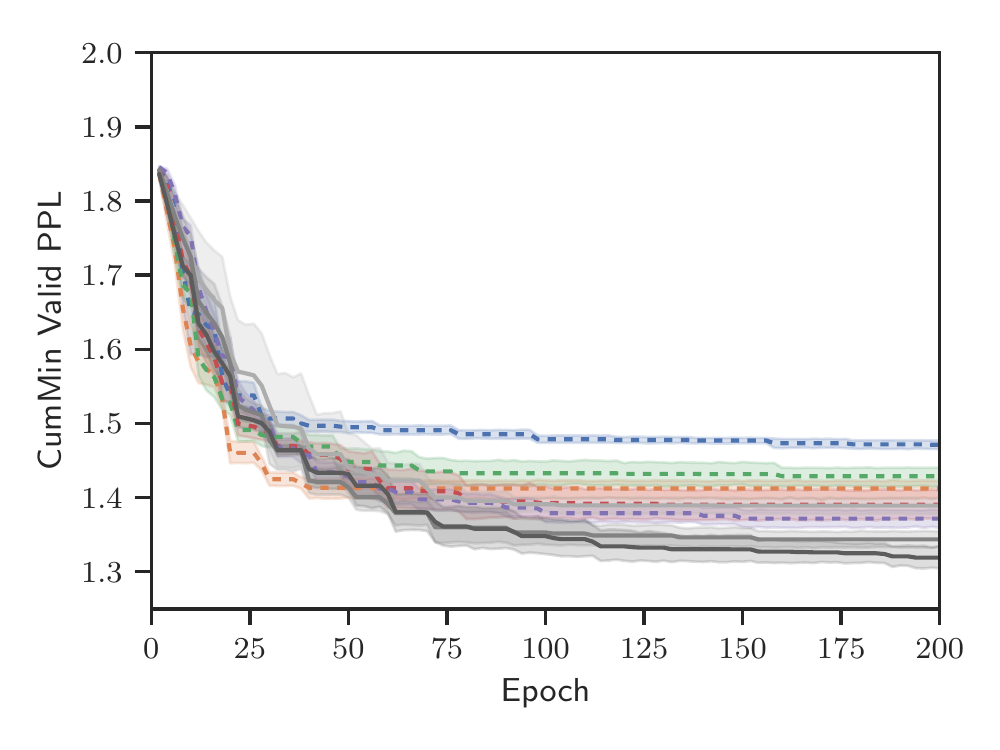}
\end{minipage}
\vspace{-1em}
\caption{Synthetic Copy Supplement: (left) Valid PPL vs epoch, (right) `Best' Valid PPL vs epoch (Top) fixed $m = 10$, (bottom) variable $m \in [5, 10]$.
Solid dark lines are our adaptive TBPTT methods, dashed colored lines are fixed TBPTT baselines.
}
\label{fig:supp_copy}
\end{figure}

\begin{table}[H]
\centering
\caption{Table of PPL for Synthetic Copy Experiments: (left) fixed $m = 10$, (right) variable $m \in [5, 10]$.
`Valid PPL' is the best validation set PPL.
`Test PPL' is the test set PPL at parameters of the best validation set PPL.
Standard deviation over multiple initializations are in parentheses.
}\label{tab:copy}
\vskip 0.15in
\begin{center}
\begin{small}
\begin{tabular}{l|rr}
\multicolumn{3}{c}{Fixed Copy $m = 10$} \\[0.1em]
\toprule
K              &   Valid PPL &   Test PPL \\
\midrule
5              & 1.655 (0.012) & 1.646 (0.012) \\
10             & 1.035 (0.007) & 1.036 (0.005) \\
15             & 1.038 (0.005) & 1.039 (0.003) \\
20             & 1.045 (0.009) & 1.040 (0.006) \\
30             & 1.044 (0.007) & 1.043 (0.004) \\
$\delta = 0.9$ & \textBF{1.018 (0.005)} & \textBF{1.022 (0.006)} \\
$\delta = 0.5$ & \textBF{1.024 (0.003)} & \textBF{1.027 (0.002)} \\
$\delta = 0.1$ & \textBF{1.029 (0.004)} & \textBF{1.030 (0.005)} \\
\bottomrule
\end{tabular}
\quad
\begin{tabular}{l|rr}
\multicolumn{3}{c}{Variable Copy $m \in [5,10]$} \\[0.1em]
\toprule
K              &   Valid PPL &   Test PPL \\
\midrule
5              & 1.46 (0.01) & 1.47 (0.01) \\
10             & 1.41 (0.02) & 1.39 (0.02) \\
15             & 1.39 (0.03) & 1.37 (0.03) \\
20             & 1.39 (0.03) & 1.35 (0.03) \\
30             & \textBF{1.33 (0.02)} & \textBF{1.31 (0.01)} \\
$\delta = 0.9$ & 1.37 (0.02) & 1.35 (0.02) \\
$\delta = 0.5$ & \textBF{1.33 (0.01)} & \textBF{1.32 (0.02)} \\
$\delta = 0.1$ & \textBF{1.31 (0.01)} & \textBF{1.29 (0.01)} \\
\bottomrule
\end{tabular}
\end{small}
\end{center}
\end{table}

\subsection{Language Modeling Experiment}
Figure~\ref{fig:supp_lm} shows the validation PPL for the two language modeling experiments in Section~\ref{sec:exp-lm}.
The left pair of figures show the validation PPL while the right pair shows the cumulative minimum (i.e. the `best') validation PPL.
The top row corresponds to the PTB experiment.
We see that fixed TBPTT with small $K$ quickly begins to over-fit (as the validation PPL increases).
With larger $K$, fixed TBPTT achieves lower validation (and test) PPL, but requires more epochs.
We see that the adaptive TBPTT with $\delta = 0.1$, achieves a better PPL much more rapidly.
The bottom row corresponds to Wiki2 experiment, where we see that the adaptive TBPTT and best fixed TBPTT method perform similarly.
Table~\ref{tab:lm} is a table of the test PPL results evaluated at the `best' validation PPL.
This table provides the numeric values of the `best' PPL values for Figures~\ref{fig:lm_results} and \ref{fig:supp_lm}.
We see that the adaptive TBPTT perform as well as or outperform the best fixed $K$ TBPTT.

\begin{figure}[H]
\centering
\begin{minipage}[c]{.32\textwidth}
    \centering
    \includegraphics[width=\textwidth]{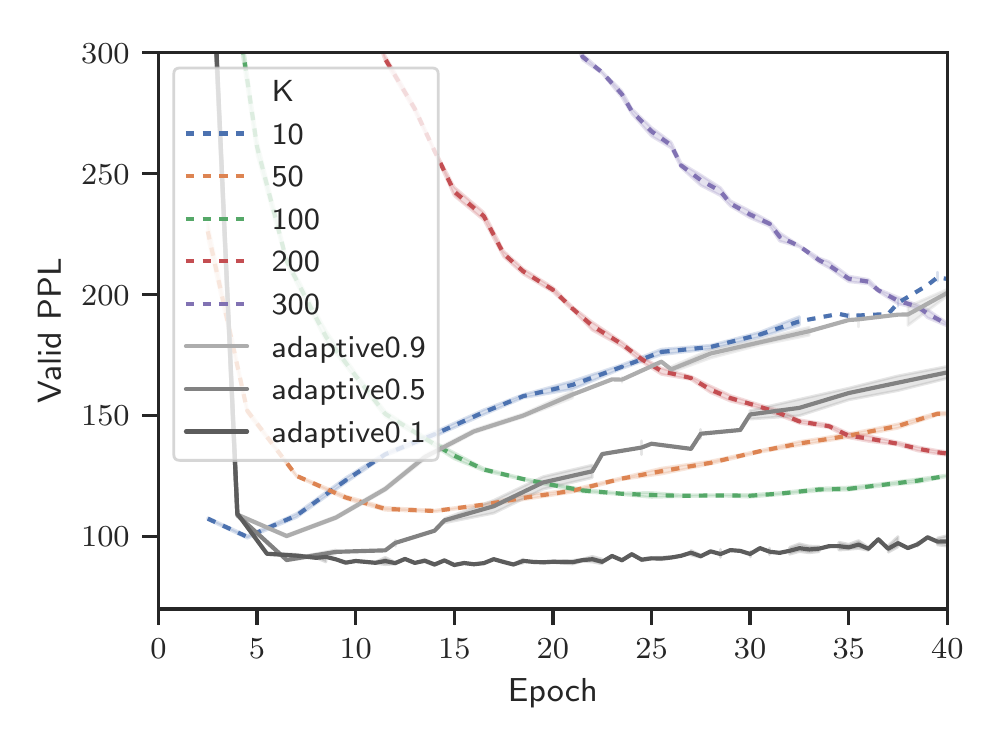}
\end{minipage}
\hspace{1em}
\begin{minipage}[c]{.32\textwidth}
    \centering
    \includegraphics[width=\textwidth]{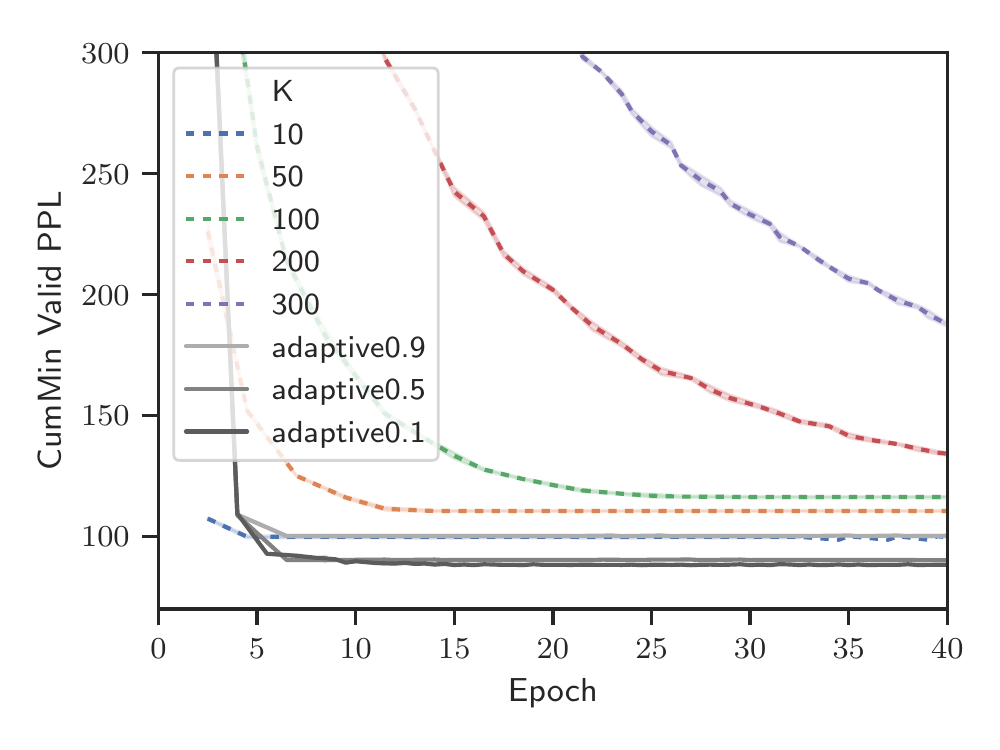}
\end{minipage}

\begin{minipage}[c]{.32\textwidth}
    \centering
    \includegraphics[width=\textwidth]{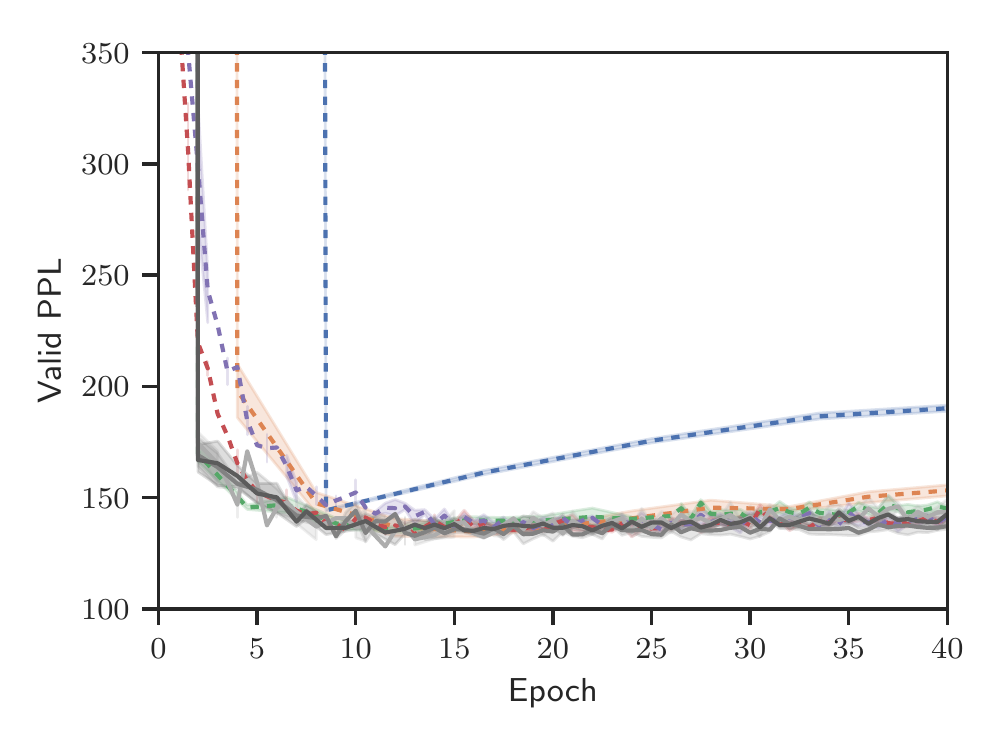}
\end{minipage}
\hspace{1em}
\begin{minipage}[c]{.32\textwidth}
    \centering
    \includegraphics[width=\textwidth]{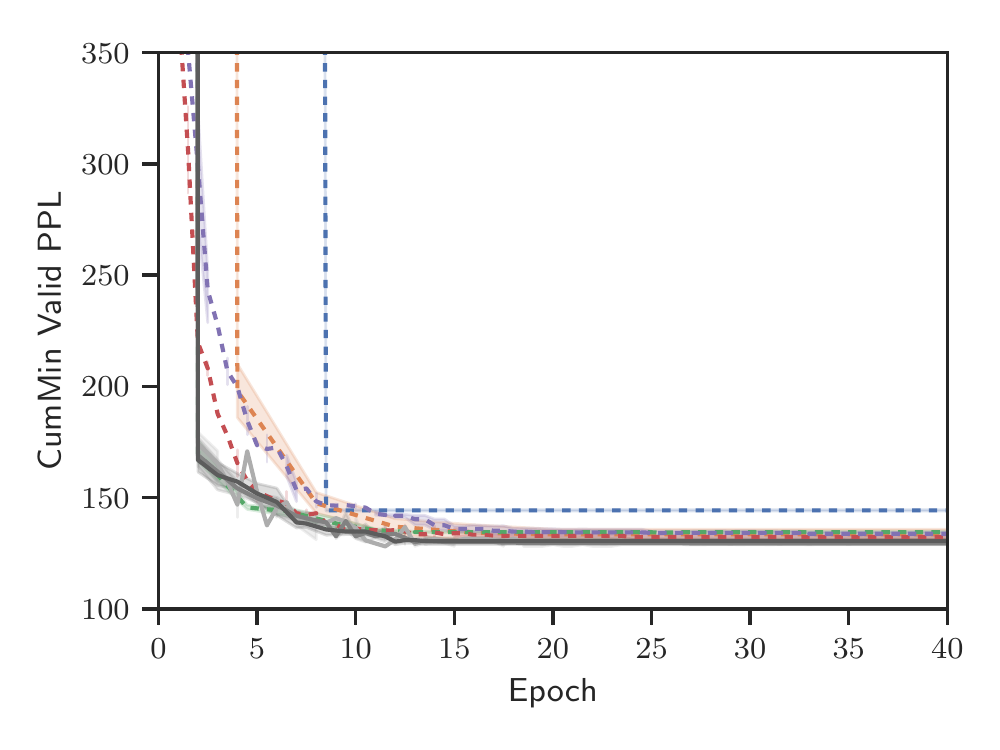}
\end{minipage}
\vspace{-1em}
\caption{Language Modeling Supplement: (left) Valid PPL vs epoch, (right) `Best' Valid PPL vs epoch. (Top) PTB, (bottom) Wiki2.
Solid dark lines are our adaptive TBPTT methods, dashed colored lines are fixed TBPTT baselines.
}
\label{fig:supp_lm}
\end{figure}
\begin{table}[H]
\caption{Table of PPL for Language Modeling experiments: (left) PTB, (right) Wiki2.
`Valid PPL' is the best validation set PPL.
`Test PPL' is the test set PPL at parameters of the best validation set PPL.
Standard deviation over multiple initializations are in parentheses.
}
\label{tab:lm}
\vskip 0.15in
\begin{center}
\begin{small}
\begin{tabular}{l|rr}
\multicolumn{3}{c}{PTB} \\[0.1em]
\toprule
K              &   Valid PPL &   Test PPL \\
\midrule
10             &   99.7 (0.6) &  99.9 (0.8) \\
50             &  110.4 (0.4) & 110.8 (0.8) \\
100            &  116.2 (0.5) & 116.9 (0.5) \\
200            &  125.2 (1.2) & 126.1 (0.9) \\
300            &  161.5 (0.5) & 161.2 (0.3) \\
$\delta = 0.9$ &  100.1 (0.5) &  99.0 (0.5) \\
$\delta = 0.5$ &   90.1 (0.4) &  89.5 (0.3) \\
$\delta = 0.1$ &   \textBF{88.1 (0.2)} &  \textBF{87.2 (0.2)} \\
\bottomrule
\end{tabular}
\quad
\begin{tabular}{l|rr}
\multicolumn{3}{c}{Wiki2} \\[0.1em]
\toprule
K              &   Valid PPL &   Test PPL \\
\midrule
10             &  144.2 (0.4) & 136.5 (1.3) \\
50             &  133.4 (2.9) & 127.2 (2.8) \\
100            &  134.4 (0.3) & 127.8 (0.5) \\
200            &  130.3 (1.1) & 124.6 (0.7) \\
300            &  \textBF{129.6 (1.4)} & \textBF{124.0 (2.2)} \\
$\delta = 0.9$ &  \textBF{130.0 (1.3)} & \textBF{124.1 (2.2)} \\
$\delta = 0.5$ &  \textBF{127.2 (0.7)} & \textBF{121.7 (0.6)} \\
$\delta = 0.1$ &  \textBF{127.5 (0.6)} & \textBF{121.9 (1.2)} \\
\bottomrule
\end{tabular}
\end{small}
\end{center}
\end{table}

\subsection{Temporal Point Process Estimation}
We now consider applying our adaptive TBPTT scheme to optimizing neural networks for temporal point prediction as in~\cite{du2016recurrent}.
Given a sequence $\{(y_i, t_i)_{i = 1}^N\}$ of categorical observations $y_i \in \mathcal{Y}$ and observation times $t_i \in \R$, the task consider by~\citep{du2016recurrent} is to predict $(y_i, t_i)$ given $(y_j, t_j)_{j < i}$.
Following~\citep{du2016recurrent}, we model the sequence using an RNN, with input embedding layers for $y_{i-1}$ and $t_{i-1}$, and two output prediction layers: one for $y_i$ and another for $\lambda(t)$ the \emph{conditional temporal point process intensity}.
The loss now consists of two terms, which define the negative log-likelihood (NLL) for a temporal point process: (i) cross entropy loss for $y_i$ and (ii) a temporal point process loss for $\lambda(t_i)$ given by Eq.(12) in~\cite{du2016recurrent}.
\citep{du2016recurrent} also evaluate the neural network model by measuring the zero-one loss of the predicted observation $\hat{y}_i$ to $y_i$ and the root mean-squared error (RMSE) of the mean predicted observation time $\hat{t}_i = \E[t_i | \lambda(t)]$ to $t_i$.

We fit such a model using a two-layer LSTM to the `Book Order' financial data used in~\cite{du2016recurrent}.
For the input layers, we use an embedding of size $128$ for the two state categorical observations $y$ and a two dimensional encoding of $t_i$ (i.e. $[t_i - t_{i-1}, t_i]$).
For the two-layer LSTM, we use a hidden and cell state dimension of size $128$.
And the output layer dimensions follow~\cite{du2016recurrent}.
For training, we use a batchsize of $S=64$ and a fixed learning rate of $\gamma = 0.1$ for SGD.
We compare gradients from fixed TBPTT $K \in [3,6,9,15,21]$ and our adaptive TBPTT method $\delta \in [0.9, 0.5, 0.1]$.
We set $W = 200$, $K_0 = 6$ and $[K_{\min}, K_{\max}] = [1,100]$ for Algorithm~\ref{alg:adaptive_tbptt}.

The `Book Order' dataset consists of the high-frequency financial transactions from the NYSE for a stock in one day. It consists of $0.7$ million transactions records (in milliseconds) and the possible actions $\mathcal{Y}$ are `to buy' or `to sell'.
We use the train-test split of~\cite{du2016recurrent} and split their test set in half to form a validation set.

The results of our experiment in Figures~\ref{fig:tpp_so} and~\ref{fig:supp_tpp_so} and Table~\ref{tab:tpp_so}.
From Figure~\ref{fig:tpp_so}(bottom center-right and bottom right) we see that the adaptive methods control for bias, by slowly increasing $K$.
From Figure~\ref{fig:tpp_so}(top right) and Table~\ref{tab:tpp_so}, we find that adaptive TBPTT methods achieve the best test set NLL.
We also see from Figure~\ref{fig:tpp_so}(bottom left and bottom center-left) and Table~\ref{tab:tpp_so} that fixed TBPTT $K = 3$ performs better at predicting $y_i$ at the cost of increased error in predicting $t_i$.
Similarly, fixed TBPTT $K = 15$ and $K=21$ are better at predicting $t_i$, but poorer at predicting $y_i$.

\begin{figure}[H]
\centering
\begin{minipage}[c]{.3\textwidth}
    \centering
    \includegraphics[width=\textwidth]{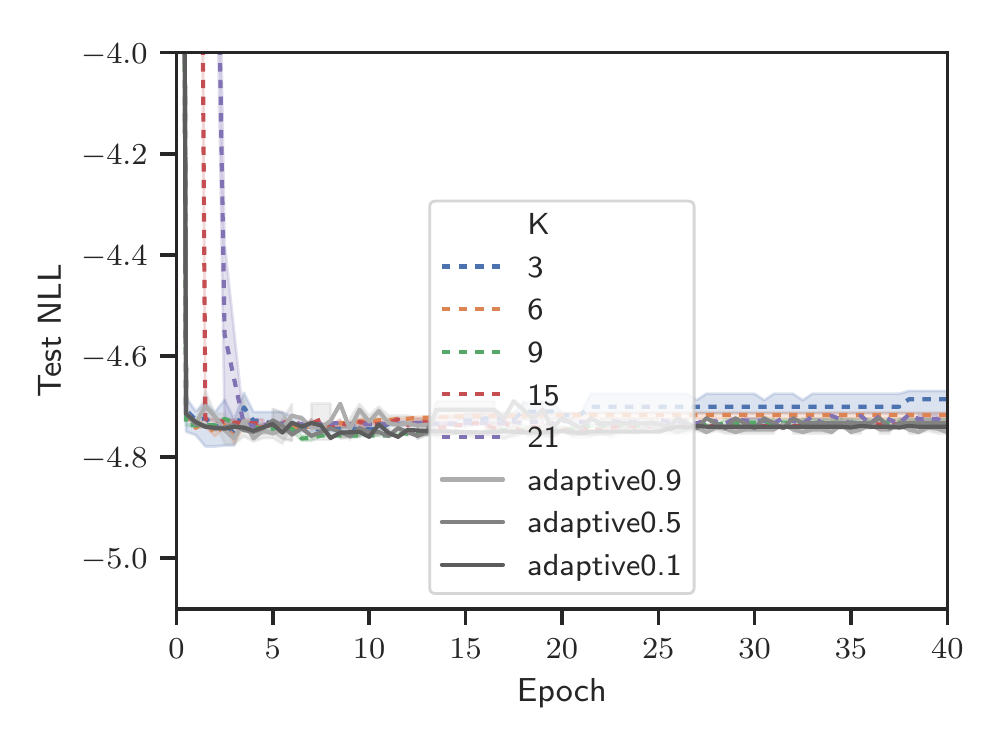}
\end{minipage}
\hspace{1em}
\begin{minipage}[c]{.3\textwidth}
    \centering
    \includegraphics[width=\textwidth]{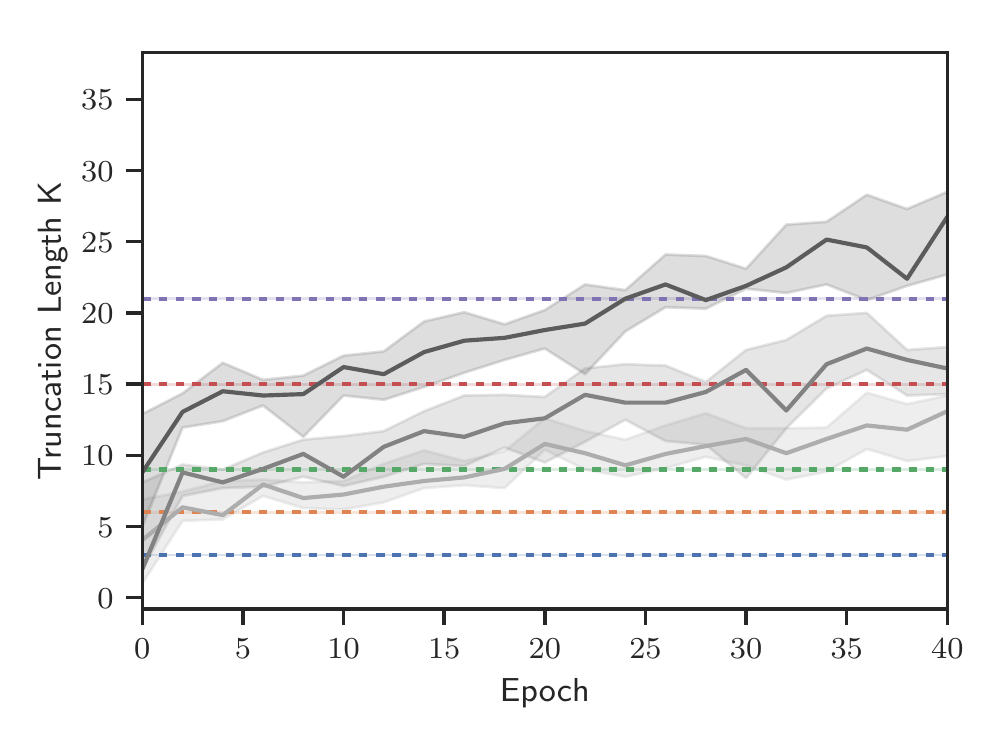}
\end{minipage}
\hspace{1em}
\begin{minipage}[c]{.3\textwidth}
    \centering
    \includegraphics[width=\textwidth]{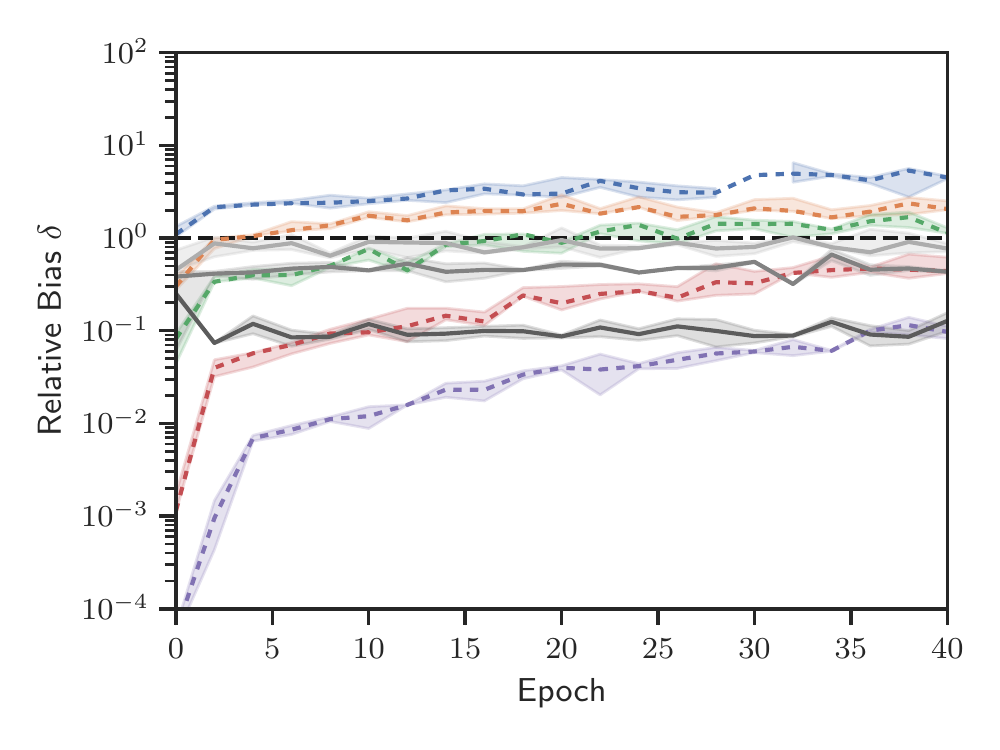}
\end{minipage}
\vspace{1em}
\begin{minipage}[c]{.3\textwidth}
    \centering
    \includegraphics[width=\textwidth]{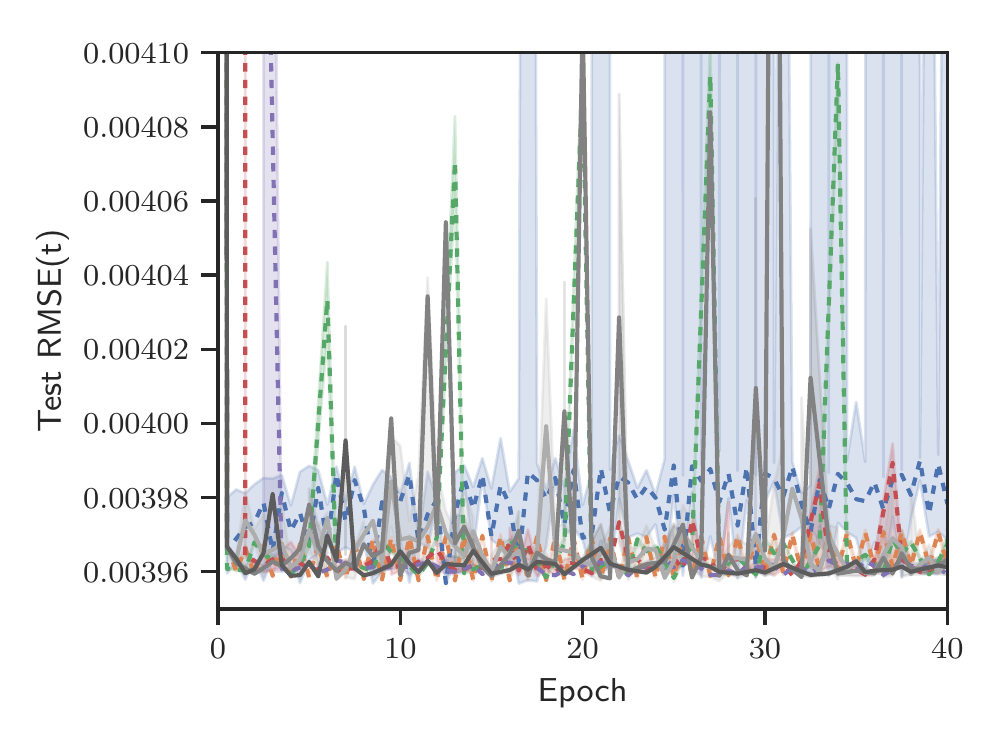}
\end{minipage}
\hspace{1em}
\begin{minipage}[c]{.3\textwidth}
    \centering
    \includegraphics[width=\textwidth]{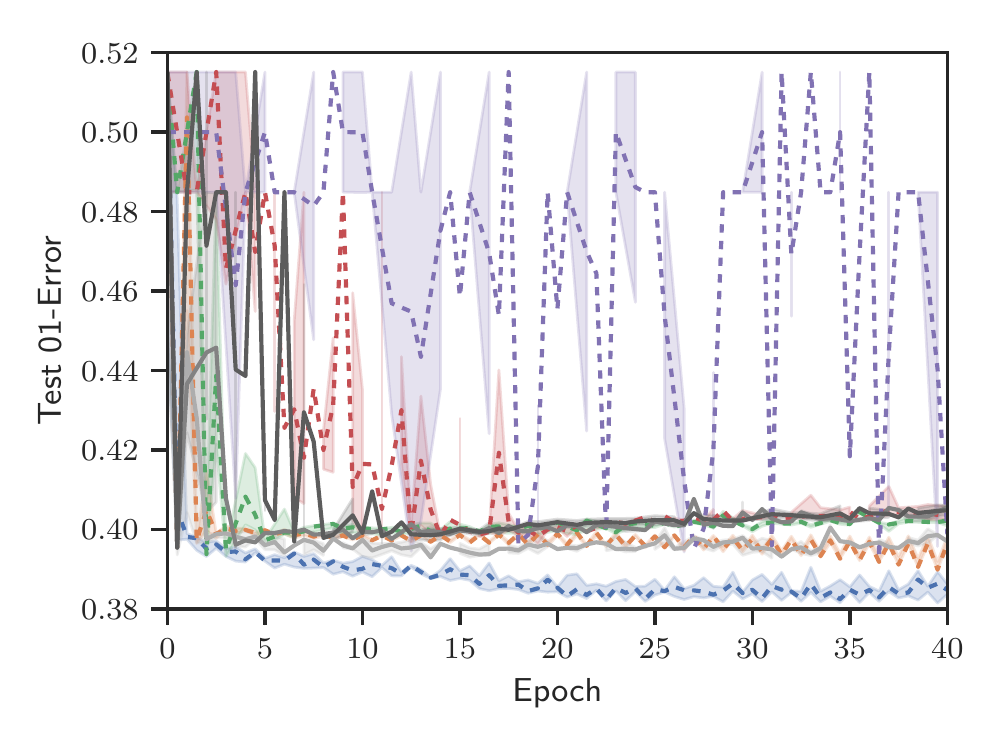}
\end{minipage}
\caption{Book Order Experiment. Top row: (left) Test NLL.
(center) truncation length $\hat{\kappa}(\delta, \theta_n)$,
(right) relative bias $\hat{\delta}(K)$.
Bottom row: (left) Test RMSE for $t$,
(right) Test 01-Error for $y$.
Solid dark lines are our adaptive TBPTT methods, dashed colored lines are fixed TBPTT baselines.
}
\label{fig:tpp_so}
\end{figure}

\begin{table}[H]
\caption{Table of metrics for Book Order experiment.
Test metrics are evaluated at the parameters of the best valdiation set NLL.
Standard deviation over multiple initializations are in parentheses.
}
\label{tab:tpp_so}
\vskip 0.15in
\begin{center}
\begin{small}
\begin{tabular}{l|rrrr}
\toprule
K              &   Valid NLL &   Test NLL &  RMSE($t$) $10^{-3}$ & 01-Loss($y$) \\
\midrule
3              &  -4.983 (0.013) & -4.694 (0.015) & 3.9705 (0.0016) & \textBF{0.3827 (0.0003)} \\
6              &  -4.905 (0.006) & -4.716 (0.006) & 3.9691 (0.0007) & 0.3959 (0.0009) \\
9              &  -4.898 (0.005) & \textBF{-4.732 (0.005)} & 3.9634 (0.0003) & 0.3944 (0.0011) \\
15             &  -4.875 (0.007) & \textBF{-4.734 (0.004)} & \textBF{3.9619 (0.0011)} & 0.3971 (0.0010) \\
21             &  \textBF{-4.831 (0.026)} & -4.719 (0.019) & \textBF{3.9622 (0.0001)} & 0.4316 (0.0336) \\
$\delta = 0.9$ &  -4.930 (0.016) & \textBF{-4.745 (0.009)} & 3.9641 (0.0007) & 0.3932 (0.0006) \\
$\delta = 0.5$ &  -4.890 (0.003) & \textBF{-4.733 (0.013)} & \textBF{3.9662 (0.0043)} & 0.3953 (0.0002) \\
$\delta = 0.1$ & \textBF{ -4.867 (0.001)} & \textBF{-4.739 (0.002)} & 3.9634 (0.0003) & 0.3954 (0.0001) \\
\bottomrule
\end{tabular}
\end{small}
\end{center}
\end{table}

\begin{figure}[H]
\centering
\begin{minipage}[c]{.3\textwidth}
    \centering
    \includegraphics[width=\textwidth]{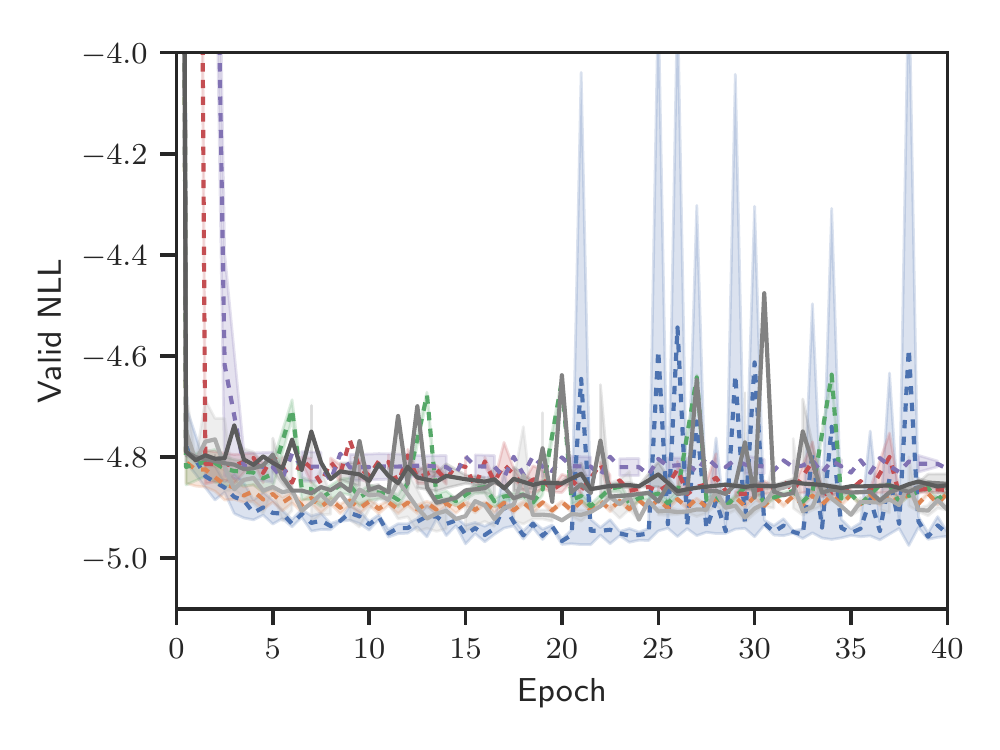}
\end{minipage}
\hspace{1em}
\begin{minipage}[c]{.3\textwidth}
    \centering
    \includegraphics[width=\textwidth]{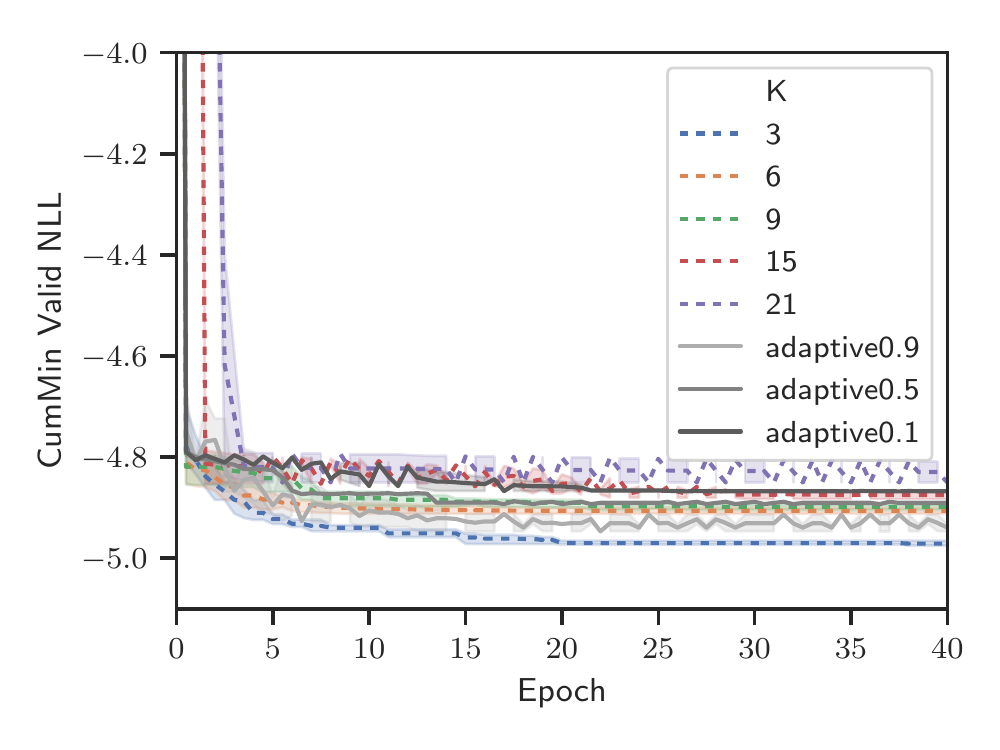}
\end{minipage}
\caption{Book Order Experiment: (left) Valid NLL, (right) `Best' Valid NLL.
Solid dark lines are our adaptive TBPTT methods, dashed colored lines are fixed TBPTT baselines.
}
\label{fig:supp_tpp_so}
\end{figure}

\end{document}